\documentclass[a4paper]{article}

\usepackage[margin=0.75in]{geometry} 

\usepackage{amsmath}
\usepackage{amsthm}
\usepackage{amssymb}
\usepackage{algorithm}
\usepackage{algpseudocode}
\usepackage{subcaption}

\DeclareMathOperator*{\argmin}{arg\,min}		
\DeclareMathOperator*{\arginf}{arg\,inf}		

\usepackage[utf8]{inputenc}
\usepackage{hyperref}
\hypersetup{
	unicode,
	pdfauthor={Author One, Author Two, Author Three},
	pdftitle={A simple article template},
	pdfsubject={A simple article template},
	pdfkeywords={article, template, simple},
	pdfproducer={LaTeX},
	pdfcreator={pdflatex}
}

\usepackage[style=ieee, url=false, doi=false, eprint=false, isbn=false, minnames=1, maxnames=2]{biblatex}

\theoremstyle{plain}
\newtheorem{theorem}{Theorem}

\newtheorem{proposition}[theorem]{Proposition}

\theoremstyle{definition}
\newtheorem{definition}[theorem]{Definition}

\usepackage{graphicx, color}
\graphicspath{{fig/}}

\usepackage{algorithm, algpseudocode} 
\usepackage{mathrsfs} 

\usepackage{lipsum}

\title{Geometric Visual Servo Via Optimal Transport}
\author{Ethan Canzini$^1$\thanks{The work of Ethan Canzini was supported by the EPSRC ICASE Award with Airbus UK and was co-sponsored by the University of Sheffield and Airbus UK Assembly Technologies. This work received funding from the UKRI EPSRC through the Made Smarter Innovation-Research Centre for Connected Factories (Grant EP/V062123/1) and in part by the RAEng/Airbus Research Chairs \& Senior Research Fellowships Scheme (Grant RCSRF1718/5/41). Ashutosh Tiwari is the Airbus/RAEng Research Chair in Digitization for Manufacturing at the University of Sheffield.} \and Simon Pope$^2$ \and Ashutosh Tiwari$^1$}

\date{
	$^1$School of Mechanical, Aerospace \& Civil Engineering,\\ University of Sheffield \\
	$^2$School of Electrical \& Electronic Engineering,\\ University of Sheffield \\ \texttt{\{e.a.canzini, a.tiwari, s.a.pope\}@sheffield.ac.uk}\\[2ex]%
}

\begin{document}
	\maketitle

	\begin{abstract}
		When developing control laws for robotic systems, the principle factor when examining their performance is choosing inputs that allow smooth tracking to a reference input. In the context of robotic manipulation, this involves translating an object or end-effector from an initial pose to a target pose. Robotic manipulation control laws frequently use vision systems as an error generator to track features and produce control inputs. However, current control algorithms don't take into account the probabilistic features that are extracted and instead rely on hand-tuned feature extraction methods. Furthermore, the target features can exist in a static pose thus allowing a combined pose and feature error for control generation. We present a geometric control law for the visual servoing problem for robotic manipulators. The input from the camera constitutes a probability measure on the 3-dimensional Special Euclidean task-space group, where the Wasserstein distance between the current and desired poses is analogous with the geometric geodesic. From this, we develop a controller that allows for both pose and image-based visual servoing by combining classical PD control with gravity compensation with error minimization through the use of geodesic flows on a 3-dimensional Special Euclidean group. We present our results on a set of test cases demonstrating the generalisation ability of our approach to a variety of initial positions. \\
		
	\end{abstract}

	
	\section{Introduction}
	\label{sec:introduction}

    As the proliferation of robotics across a variety of fields increases, the development of feedback control laws remains a key research challenge. Many robotic platforms rely on vision systems to produce stable feedback laws \cite{stochastic_jigs_1982,vision_review_2015,safe_learning_2024}, as this sensing modality leverages a balance between cost, efficiency, and safety in a variety of environments and applications \cite{sensor_review_2021}. Connecting the vision input to a stable feedback law remains a key challenge in developing autonomous systems that are capable of dealing with uncertainty in their environmental observations. \\
    
    Visual servoing refers to a wide domain of control architectures that generate an error signal based on the tracking of image features of components. By relating the visual features to a desired pose, this form of control allows for the computation of error signals that incorporate either the pose or the camera features and track to the desired goal \cite{handbook_robotics_2008}. Visual servoing-based control remains a key part of the robotics control community, with applications ranging from manipulators \cite{vs_mpc_2019}, underwater autonomous vehicles \cite{vs_underwater_2022} and spacecraft orbital rendezvous \cite{vs_spacecraft_2024}, where the interaction matrix $\mathbf{L}_e$ can be tuned depending on the environment that the robot is operating in. However, these extracted features are purely based on camera parameters which can lead to noisy estimations of the camera pose and requires human and domain-specific knowledge to fine-tune which specific features are being tracked \cite{2d_vs_1999,vs_review_2023}. Furthermore, the use of visual servo control with system dynamics requires the use of inverse kinematics and coordinate-level proportional-derivative control which can infringe on performance guarantees \cite{comp_visual_servo_2023}. Additionally, many system dynamics lose the benefits of the physical characteristics of the robotic platforms during the derivation process \cite{history_manip_2022}. Generally speaking, visual servoing takes two different forms:
    \begin{itemize}
    	\item \textbf{Pose-Based Visual Servo (PBVS)} which relates the extracted features to a specific pose and computes the difference between the desired and current pose
    	\item \textbf{Image-Based Visual Servo (IBVS)} where the extracted features are tracked during control towards the desired feature set.
    \end{itemize}
    These features are normally extracted using geometric representations of the images and therefore lack the the granularity and feature set that is required for high-accuracy tasks \cite{vs_image_moments_2004}. For vision-based feedback, the goal is to minimise the error $e(t)$ defined as:
    \begin{equation}\label{eq:vs}
    	e(t) = s(m(t),a) - s^\ast,
    \end{equation}
    \noindent where $m(t)$ represents the camera measurements and $a$ are other intrinsic camera properties that are based on pre-defined photographic parameters. Note that, compared to most standard feedback control literature, the error is calculated in the reverse manner. For most approaches, including the one put forward in this work, the target camera features $s^\ast$ are kept constant and motionless during operation. \\
    
    Visual servoing has a rich research history in the field of robot control, with applications ranging from terrestrial robotics in manufacturing environments \cite{vs_pipe_2019, vs_aero_2020}, to controlling underwater autonomous vehicles \cite{vs_underwater_2022}, to endoscopic surgery \cite{vs_medical_2014} and the docking and manipulation of non-cooperative space satellites \cite{vs_spacecraft_2024, vs_noncooperative_2024}. Initial work was focused on better ways to extract and track features during operation, with an emphasis on using RGB images to generate and track features that are present in the frame \cite{image_moments_2004,feature_tracking_2005}. As the demand for more accurate feature tracking increased, multi-camera systems emerged as a way to improve the feature accuracy of the tracking system \cite{vs_multi_camera_2011}. However, the lack of accurate distance estimation from monocular camera fusion remained a key problem. The proliferation and improvement in RGB-D cameras allowed the development of algorithms that utilise point clouds as features, allowing the leveraging of distance data to create a depth map of features \cite{vs_depth_map_2014, vs_point_cloud_2020, vs_image_based_2022}. \\
    
    The evolution of visual servoing control has led to its formulation as an optimal control problem. This has led to its use within model predictive control (MPC) frameworks, with further implementations looking to embed robust control to improve the accuracy under feature uncertainty \cite{vs_mpc_2019,mpc_quad_2020,robust_vs_2016}. Control barrier functions provide an additional level of robustness to the visual servo control \cite{diffopt_2024}, but struggle to compute optimal feedback gains in high-dimensional environments. Machine learning (ML) methods seek to alleviate this curse of dimensionality with self-supervised learning and human demonstrations providing reference trajectories to generate optimal controls \cite{dome_2022,vs_ssl_2022}. \\
    
    Recent advances in the use of depth camera technology have enabled the development of control laws that use depth inputs to generate point clouds of physical features. This proves beneficial when dealing with complex geometries that are difficult to define or when working in cluttered environments \cite{vs_point_cloud_2020,deep_vs_2022}. These point clouds can be defined as probability measures, whereby metrics can be evaluated across them to find displacements across spaces \cite{ot_change_detect_2023,ot_metric_point_2024}. Optimal transport (OT) allows the computation of distance metrics between these measures \cite{ot_book_2015}, and can be used to generate a variety of different motion paths for autonomous systems \cite{ot_motion_2023}. More recently, thanks to the landmark work by Benamou and Brenier \cite{ot_bb_2000}, OT can be formulated dynamically to generate flow maps that are represented by a set of partial differential equations. This has seen a prevalence in control theory applications, where it can be beneficial to represent states as probability measures and find the optimal coupling between them \cite{ot_applications_2021}. This approach can be seen as an energy-based control law \cite{ot_control_2021}, and can be used for averaging control of particles \cite{ot_averaged_2023} and for steering control of stochastic linear systems \cite{ot_steering_2016,ot_linear_2017}. Furthermore, due to the nature in which optimal transport is defined, the theory lends itself to geometrical representations of the underlying dynamics and can be used for systems where geometry is crucial \cite{ot_curvature_chapter_2011,ot_manifolds_2024}. \\
    
    When working with systems that operate on geometric spaces, it is beneficial to derive the dynamics such that the geometric features are maintained. Many dynamical systems operate on the basis that their mechanical characteristics can be defined using geometric qualities such as finite rotations and positions \cite{note_rot_1989}. Lie theory provides a robust set of tools from group theory and differential geometry that can be used to define the kinematics of multi-body mechanical systems such that their geometric qualities are maintained \cite{screw_lie_2018}. Using Lie theory, system dynamics can be represented and interpreted in such a way that the kinematics of the system are formulated using geometric qualities, thus allowing the exploitation of relationships when generating control signals \cite{global_hamilton_2018,contact_rich_2024}. Hamiltonian mechanics are an alternative to classical mechanics by allowing the derivation of dynamical systems in terms of their conjugate momenta whilst preserving geometric representations \cite{intro_mechanics_1999}. For many control problems, the target dynamics can be said to exist on the 3-dimensional Special Euclidean group $\mathtt{SE}(3)$, which is composed of rotations and positions which in turn can be embedded into the dynamics, thus providing a rich framework for developing control laws. Recent work in robotics has shown that Hamiltonian mechanics allows the use of energy-based controllers to steer systems towards reference inputs \cite{robot_hamilton_2016,hamilton_pendulum_2024}, with the added benefit that the geometric representation can be maintained throughout the dynamics \cite{duong_hamiltonian_2021}. Reformulating the dynamics in port-Hamiltonian form leverages energy-conservation methods by providing a physical interpretation of the desired control law and outlining the interconnection between mechanical sub-systems \cite{port_hamilton_overview_2014,port_hamilton_2020}. \\
    
    \begin{figure*}[t]
    	\centering
    	\includegraphics[width=\linewidth]{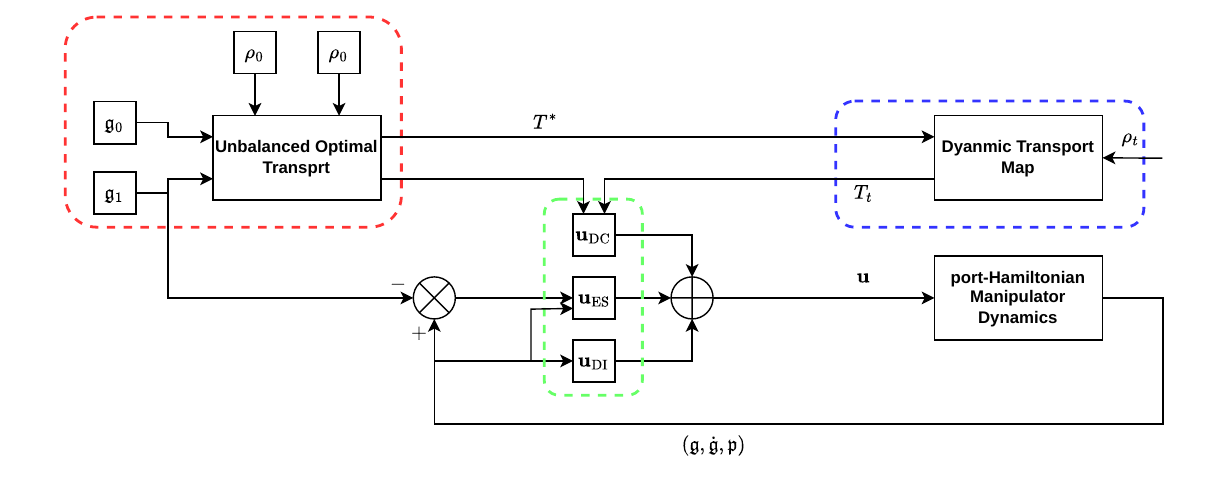}
    	\caption{Our proposed control architecture for geometric visual servoing. The computation of the optimal transport map $T^\ast$, shown in red, is performed during the initialisation of the control loop using the initial depth map and pose (Section \ref{sec:sub:features}). The transport map is then used online during the computation of the disturbance correction input $\mathbf{u}_{\text{DC}}$ using the current camera depth map, highlighted in blue (Section \ref{sec:sub:energy-ot}). The input is then combined with the passivity-based gains (Section \ref{sec:sub:closed-loop}) to compute the input to the manipulator.}
    	\label{fig:our-method}
    \end{figure*}
    
    In this work, we provide a reformulation of visual servoing by examining the nature of the extracted features from the camera through a geometric lens of the system dynamics. To alleviate the reliance on either use PBVS or IBVS methods when computing tracking errors, we develop a combined framework that can track both pose and image errors without the need to hand-craft specific features from images. Using the work established in \cite{duong_hamiltonian_2021}, we derive port-Hamiltonian dynamics on the $\mathtt{SE}(3)$ Lie group that allows for the use of an energy-based controller. By defining the extracted features from the camera as a probabilistic depth map, we apply dynamic optimal transport to insert energy into the dynamical system to steer towards the desired depth map and demonstrate how this minimum-energy task-space controller can be verified using the geometry of the Lie group. We validate our work on real-world hardware and provide future directions for research. Our method is shown in figure \ref{fig:our-method} and our contributions are as follows:                   
    \begin{enumerate}
        \item A formulation of the system dynamics of a robot manipulator in the port-Hamiltonian form, combining the task-space geometry with the joint-space control 
        \item Using the output of an RGB-D camera to generate a depth map, which in turn is described as a probability measure on the Lie group pose of the manipulator
        \item By leveraging the geometry of the Lie group, we redefine the error based on the reference depth map and use this error to generate a control signal to steer toward the desired depth map.
    \end{enumerate}
    An overview of our control algorithm is provided in figure \ref{fig:our-method}: highlighted in red, we compute the unbalanced optimal transport map from the initial depth cloud and the desired depth cloud. This is then used during the algorithm's operation in the dynamic transport control generation, highlighted in blue, which uses the transport map at each time step to compute the control input to align the current and target depth maps. This is then combined with the geometric pose control generation, highlighted in green, to provide a complete treatment of the visual servoing problem in both the pose and image spaces. Within this work, we focus on the application of robotic manipulators for visual servoing control. However, the approach described henceforth could be applied to any autonomous system that utilises vision-based control. The outline of this paper is as follows: section \ref{sec:prelim} provides a brief overview of the technical background for the proposed approach; section \ref{sec:problem} outlines the problem statement for the visual servo control; section \ref{sec:depth-dynamics} derives the dynamics for the manipulator and the depth cloud probability measures, then illustrates the use of the passivity-based control on $\mathtt{SE}(3)$ for transporting the depth map to the desired pose; section \ref{sec:results} provides our results on several test-cases that illustrate the effectiveness of our control solution; finally, section \ref{sec:conclusion} summarises our contributions and outlines what future work could be done in this field.
    
    \section{Problem Formulation}\label{sec:problem}
    
    Consider a robot manipulator with the task of moving a camera from an initial pose in $\mathtt{SE}(3)$ to a target pose. The camera provides a depth map input distribution $\mu_0 \in \mathcal{P}(\mathtt{SE}(3))$ which needs to be controlled to a distribution on the target pose $\mu_1 \in \mathcal{P}(\mathtt{SE}(3))$. We consider the instance where the vision system is mounted to the end-effector of the robot manipulator - known as the eye-in-hand configuration - such that during operation, the motion of the camera traces a path on the $\mathtt{SE}(3)$ Lie group. \\
    
    Let $\mathcal{Q} \subset \mathbb{R}^n$ denote the robot configuration space, where $n$ corresponds to the robot's degrees of freedom and the vector of joint positions is denoted as $\mathfrak{q} \triangleq [q_1,...,q_n] \in \mathcal{Q}$. We define the robots nonlinear dynamics under the input $\mathbf{u} \in \mathcal{U}$ as
    \begin{equation}
    	\mathbf{x}_{t+1} = f(\mathbf{x}, \mathbf{u}),
    \end{equation}
    where the task-space dynamics capture the geometry of the end-effector and the configuration space of the manipulator, allowing the control of the eye-in-hand camera toward the desired depth map. The goal is to find the optimal control input that minimises the distance error between the target and current depth frames\footnote{Note that in opposition to \ref{eq:vs}, we return to the optimal control notation style when computing the error between depth maps.}:
    \begin{equation}
    	\mathbf{u}^\ast = \argmin_{\mathbf{u} \in \mathcal{U}} \Vert s^\ast - s_t \Vert^p_p.
    \end{equation}
    where $p$ denotes the norm of the error signal. 
	
	\section{Preliminaries}\label{sec:prelim}

    \subsection{$\mathtt{SE}(3)$ Matrix Lie Group}
    
    We first provide some background and guidance on notation for the $\mathtt{SE}(3)$ Lie group. We refer the reader to \cite{geo_control_2005} and \cite{intro_mechanics_1999} for a more detailed and general overview. We begin by providing some definitions for our use of Lie algebras.
    
    \begin{definition}\label{def:pairing}
    Let $\xi$ and $\varphi$ $\in$ $\mathbb{R}^{n\times m}$. The dot product 
    $\langle \cdot, \cdot \rangle$ can be chosen as \cite{duong_hamiltonian_2021}: 
    \begin{equation}
        \langle \xi, \varphi \rangle = tr(\xi^\top\varphi).
    \end{equation}
    \end{definition}
    
    \begin{definition}
    For a given pose $\boldsymbol{g} \in \mathtt{SE}(3)$, the adjoint $\mathtt{Ad}_{\boldsymbol{g}} : \boldsymbol{g} \rightarrow \boldsymbol{g}$ is defined as:
    \begin{equation}
        \mathtt{Ad}_{\boldsymbol{g}} (\psi) = \boldsymbol{g}\psi\boldsymbol{g}^{-1}.
    \end{equation}
    For the Lie algebra $\mathfrak{se}(3)$, the directional derivative of $\mathtt{Ad}_{\boldsymbol{g}}$ in the direction of $\psi \in \mathtt{SE}(3)$ is the algebra adjoint $\mathtt{ad}_{\boldsymbol{g}}$ which is defined as \cite{screw_lie_2018}:
    \begin{equation}\label{eq:adjoint}
        \mathtt{ad}_{\boldsymbol{g}} (\psi) = \frac{\text{d}}{\text{d}t} \mathtt{Ad}_{\exp_{\mathtt{SE}(3)}(t\boldsymbol{g})}(\psi) \bigg|_{t=0} = [\boldsymbol{g}, \psi],
    \end{equation}
    where $[\cdot, \cdot] : \mathfrak{se}(3) \times \mathfrak{se}(3) \rightarrow \mathfrak{se}(3)$ is the Lie bracket operator \cite[Def. 5]{duong_hamiltonian_2021} and $\exp_{\mathtt{SE}(3)}(\hat{\cdot})$ is the exponential map from $\mathtt{SE}(3)$ to $\mathtt{R}^3 \bigoplus \mathtt{R}^3$. 
    \end{definition}
    
    Let $\boldsymbol{g} \in \mathbb{R}^3 \times \mathbb{R}^3$ be a pose of the end effector in a three-dimensional Euclidean space combining the rotational elements $\mathbf{R}(t) \triangleq [\mathbf{r}_1, \mathbf{r}_2, \mathbf{r}_3]^\top$ in the Special Orthogonal group $\mathtt{SO}(3)$ and positional elements $\mathbf{p}(t) \in \mathbb{R}^3$. For any rigid body trajectory evolving in time $t \mapsto \boldsymbol{g}$, we denote its membership of the special Euclidean Lie group $\mathtt{SE}(3)$ as:
    \begin{equation} \label{eq:se3}
    \boldsymbol{g} \triangleq \left \{ \left [\begin{array}{cc}
        \mathbf{R}(t) & \mathbf{p}(t)\\
        \mathbf{0}_{1\times3} & 1
    \end{array}\right ] \in \mathtt{SE}(3)\, \bigg| \, \mathbf{R}(t) \in \mathtt{SO}(3),\, \mathbf{p}(t) \in \mathbb{R}^3 \right \}.
    \end{equation}
    
    The body of the end effector has a velocity that moves along its trajectory $t \mapsto \boldsymbol{\xi}(t) = (\boldsymbol{v}, \boldsymbol{\omega}) \in \mathbb{R}^3 \bigoplus \mathbb{R}^3$. Using an abuse of notation, we can define the isomorphism $\hat{(\cdot)}: \mathbb{R}^3 \bigoplus \mathbb{R}^3 \rightarrow \mathfrak{se}(3)$ that defines the vector space Lie algebra of $\mathtt{SE}(3)$ \cite{intro_mechanics_1999}. 
    \begin{definition}
        A twist $\hat{\boldsymbol{\xi}}$ consists of the rotational and positional velocities and is written as:
        \begin{equation}
        \hat{\boldsymbol{\xi}} \triangleq \left \{ \left [\begin{array}{cc}
            \hat{\boldsymbol{\omega}}(t) & \boldsymbol{v}(t)\\
            \mathbf{0}_{1\times 3} & 0
        \end{array}\right ] \in \mathfrak{se}(3)\, \bigg| \, \hat{\boldsymbol{\omega}}(t) \in \mathfrak{so}(3),\, \boldsymbol{v}(t) \in \mathbb{R}^3 \right \},
        \end{equation}
        
        \noindent where:
        \begin{equation}
        \hat{\boldsymbol{\omega}}(t) = \left [ \begin{array}{ccc}
            0 & -\omega_z & \omega_y \\
            \omega_z & 0 & -\omega_x \\
            -\omega_y & \omega_x & 0
        \end{array} \right ]
        \end{equation}
        \noindent is the skew-symmetric rotational velocity matrix. We denote $(\cdot)^\vee: \mathfrak{se}(3) \rightarrow \mathbb{R}^3 \bigoplus \mathbb{R}^3$ as the inverse isomorphism that maps from the $\mathtt{SE}(3)$ Lie algebra to the vector space. 
    \end{definition}
    For notational simplicity, we say $\hat{\boldsymbol{\xi}}(t) \equiv \hat{\boldsymbol{\xi}}$ for all time-dependent signals. \\
    
    Rodriguez's formula for unit twists allows the efficient computation of the rotations across twists \cite{note_rot_1989}. It is defined as:
    \begin{equation}\label{eq:rod}
    \begin{aligned}
    e^{\hat{\boldsymbol{\xi}}q} \approx \mathbf{I}_4 + \hat{\boldsymbol{\xi}}q & = \left [\begin{array}{cc}
            e^{\hat{\boldsymbol{\omega}}q} & (\mathbf{I}_3 - e^{\hat{\boldsymbol{w}}q})(\boldsymbol{\omega}\times \boldsymbol{v}) + \boldsymbol{\omega}\boldsymbol{\omega}^\top \boldsymbol{v}q \\
            0 & 1
        \end{array} \right ], \\
        \text{where } e^{\hat{\boldsymbol{\omega}}q} & \approx \mathbf{I}_3 + \frac{\hat{\boldsymbol{\omega}}}{||\boldsymbol{\omega}_i||}\sin(||\boldsymbol{\omega}||q) + \frac{\hat{\boldsymbol{\omega}}^2}{||\boldsymbol{\omega}||^2}(1 - \cos(||\boldsymbol{\omega}||q).
    \end{aligned}
    \end{equation}
    
    \begin{definition}\label{def:left-invariant}
    Let $\boldsymbol{g} \in \mathtt{SE}(3)$ and $\boldsymbol{h} \in \mathtt{SE}(3)$. The left translation $\mathtt{L}_{\boldsymbol{g}} : \mathtt{SE}(3) \rightarrow \mathtt{SE}(3)$ is defined as:
    \begin{equation}
        \mathsf{L}_{\boldsymbol{g}}(\boldsymbol{h}) = \boldsymbol{g}\boldsymbol{h}.
    \end{equation}
    For a twist $\hat{\boldsymbol{\xi}} \in \mathfrak{se}(3)$, the kinematics of the Lie group give us the left-invariant vector that allows us to compute the velocity on the tangent space $\dot{\boldsymbol{g}} \in \mathtt{T}_{\boldsymbol{g}}\mathtt{SE}(3)$:
    \begin{equation}\label{eq:left-vector}
        \dot{\boldsymbol{g}} = \mathsf{T}_{\text{e}}\mathsf{L}_{\boldsymbol{g}}(\hat{\boldsymbol{\xi}}) = \boldsymbol{g}\hat{\boldsymbol{\xi}}.
    \end{equation}
    \end{definition}
    
    \begin{definition} \label{def:dual-map}
        Consider the pairing $\langle \cdot, \cdot\rangle$ from Definition \ref{def:pairing}. For all $\eta, \xi$ on $\mathfrak{se}(3)^\ast \times \mathfrak{se}(3)$, the dual map $\mathsf{T}^\ast_e\mathsf{L}_{\mathsf{g}}$ satisfies:
        \begin{equation}
            \langle \mathsf{T}^\ast_e\mathsf{L}_{\mathfrak{g}}(\eta), \xi \rangle = \langle \eta, \mathtt{T}_e\mathsf{L}_{\mathsf{g}}(\xi) \rangle.
        \end{equation}
    \end{definition}
    
    \subsection{Hamiltonian Mechanics}
    
    Hamiltonian mechanics are a representation of the Lagrangian mechanics of a system using the generalized coordinates and the conjugate momenta \cite{intro_mechanics_1999}. We provide here an overview of the physics that we will use in our methodology, and refer the reader to \cite{global_hamilton_2018} and \cite{intro_mechanics_1999} for more details. \\
    
    Let $\boldsymbol{q} \in \mathcal{Q}$ be the generalized coordinate in the configuration space with $\dot{\boldsymbol{q}} \in \mathsf{T}\mathcal{Q}$ being the generalized velocity on the tangent space of $\mathcal{Q}$. The Lagrangian for the desired state space $\mathbf{x} = (\boldsymbol{q}, \dot{\boldsymbol{q}}) \in \mathcal{X} \subset \mathsf{T}\mathcal{Q}$ is defined as the difference between the kinetic and potential energies:
    \begin{equation}\label{eq:lagrange}
    \mathcal{L}(\boldsymbol{q}, \dot{\boldsymbol{q}}, t) = \mathcal{K}(\boldsymbol{q}, \dot{\boldsymbol{q}}) - \mathcal{G}(\boldsymbol{q}).
    \end{equation}
    To convert from Lagrangian coordinates to Hamiltonian phase space $(\boldsymbol{q}, \dot{\boldsymbol{q}}) \mapsto (\boldsymbol{q}, \boldsymbol{p})$, we use the Legendre transform.
    \begin{definition}\label{def:hamilton}
    The Hamiltonian $\mathcal{H}(\boldsymbol{q}, \boldsymbol{p}, t)$ is defined from the Lagrangian using the Legendre transform:
    \begin{equation}\label{eq:legendre-hamilton}
        \mathcal{H}(\boldsymbol{q}, \boldsymbol{p}, t) \triangleq \sup_{\dot{\boldsymbol{q}} \in \mathsf{T}\mathcal{Q}}\, \langle \boldsymbol{p}, \dot{\boldsymbol{q}} \rangle - \mathcal{L}(\boldsymbol{q}, \dot{\boldsymbol{q}}, t)
    \end{equation}
    where the phase space-time derivatives are defined as:
    \begin{equation}\label{eq:phase-space}
        \dot{\boldsymbol{q}} = \frac{\partial \mathcal{H}}{\partial\boldsymbol{p}}, \ \dot{\boldsymbol{p}} = - \frac{\partial \mathcal{H}}{\partial\boldsymbol{q}} + \tau,
    \end{equation}
    where $\tau$ represents the generalised forces in the system. The phase space coordinates act on the cotangent bundle $(\boldsymbol{q}, \boldsymbol{p}) \in \mathcal{X}^* \subset \mathtt{T}^*\mathcal{Q}$. 
    \end{definition}
    
    Definition \ref{def:hamilton} allows the formulation of the dynamics in a port-based form to ensure energy conservation \cite{schaft_port_2004}. 
    \begin{definition}\label{def:port-hamil}
    For a phase space $\mathbf{x} = (\boldsymbol{q}, \boldsymbol{p}) \in \mathcal{X}^\ast$, the dynamics $\dot{\mathbf{x}} = (\dot{\boldsymbol{q}}, \dot{\boldsymbol{p}}) \in \mathsf{T}\mathcal{X}^\ast$ satisfy the port-Hamiltonian structure:
    \begin{equation}
        \dot{\mathbf{x}} = \left [ \mathcal{J}(\boldsymbol{q}, \boldsymbol{p}) - \mathcal{R}(\boldsymbol{q}, \boldsymbol{p}) \right]\nabla\mathcal{H}(\boldsymbol{q}, \boldsymbol{p}, t ) + \mathcal{B}(\boldsymbol{q}, \boldsymbol{p})\mathbf{u},
    \end{equation}
    where $\mathcal{J}(\boldsymbol{q}, \boldsymbol{p})$ is the skew-symmetric interconnection matrix which represents energy storage elements and $\mathcal{R}(\boldsymbol{q}, \boldsymbol{p})\succeq 0$ is the energy dissipation matrix. 
    \end{definition}
    
    \begin{definition}\label{def:lie-poisson-bracket}
        Let $\mathcal{F}$ be a function in the phase-space $(\boldsymbol{q}, \boldsymbol{p})$. The time derivative $\dot{\mathcal{F}}$ can be defined using the Poisson bracket:
        \begin{equation}
            \begin{aligned}
                \dot{\mathcal{F}} = \left \{ \mathcal{F}, \mathcal{H} \right \}& = \frac{\partial\mathcal{F}}{\partial t} + \sum_{i=1}^n \left( \frac{\partial \mathcal{F}}{\partial \boldsymbol{q}_i}\frac{\partial\mathcal{H}}{\partial\boldsymbol{p}_i} - \frac{\partial \mathcal{F}}{\partial \boldsymbol{p}_i}\frac{\partial\mathcal{H}}{\partial\boldsymbol{q}_i}\right) \\
                & = \frac{\partial\mathcal{F}}{\partial t} + \left\langle \frac{\partial \mathcal{F}}{\partial \boldsymbol{q}},\frac{\partial\mathcal{H}}{\partial\boldsymbol{p}} \right\rangle - \left\langle \frac{\partial \mathcal{F}}{\partial \boldsymbol{p}},\frac{\partial\mathcal{H}}{\partial\boldsymbol{q}} \right\rangle,
            \end{aligned}
        \end{equation}
        where $\mathcal{H}$ denotes the Hamiltonian. 
    \end{definition}
    
    \subsection{Optimal Transport}
    
    For a more complete reading on the applications of optimal transport in control systems, we refer the readers to \cite{ot_control_2021} and \cite{ot_book_2015}. We start by defining the Monge problem for probability measures. 
    
    \begin{definition}\label{def:monge-problem}
    Let $\mu, \nu \in \mathcal{P}(\mathbb{R}^n)$ be $n$-dimensional probability measures. Provided there exists a strongly convex cost function $c : \mathcal{X} \times \mathcal{Y} \rightarrow \mathbb{R}$, the optimal transportation map $T^*$ that moves $\mathcal{X}$ to $\mathcal{Y}$ is defined as:
    \begin{equation}\label{eq:trans-map}
        T^* \triangleq \arginf_{T\sharp\mu=\nu} \int_{\mathbb{R}^n} c(x, T(x))\, \text{d}\mu(x),
    \end{equation}
    where $T\sharp$ is the push-forward operator that moves the initial measure to the target measure such that $T\sharp\mu = \nu$. 
    \end{definition}
    \begin{definition}\label{def:kantorovich}
    Let $\Pi(\mu, \nu) = \{ \pi \in \mathcal{P}(\mathcal{X} \times \mathcal{Y}) : P_{\mathcal{X}\sharp}\pi = \mu,\, P_{\mathcal{Y}\sharp}\pi = \nu \}$ be the set of couplings on $\mathbb{R}^n \times \mathbb{R}^n$. The Wasserstein distance between the two distributions $\mu, \nu$ is defined as:
    \begin{equation}\label{eq:wasserstein}
        \mathbb{W}(\mu, \nu) \triangleq \inf_{\pi \in \Pi(\mu, \nu)} \iint c(x, y)\pi(x,y)\, \text{d}x\text{d}y.
    \end{equation}
    \end{definition}  
    Definitions \ref{def:monge-problem} and \ref{def:kantorovich} are both static formulations of the optimal transport problem. For control theoretic problems, the transport problem can be re-framed as a dynamic transport problem. Using the Eulerian definition of optimal transport theory, the dynamic formulation of optimal transport can be framed as a stochastic optimal control problem to choose a control input :
    \begin{align}
        \inf_{u \in \mathcal{U}} \ & \mathbb{E} \left \{ \int^{1}_{0} \left\|u(t)\right\|^2 \text{d}t \right \}, \label{eq:sub:stoch-ot}\\ 
        \text{s.t. }\dot{x}(t) = &\, \mathbf{A}x(t) + \mathbf{B}u(t) + \sqrt{\epsilon}\mathbf{B}\omega(t), \label{eq:sub:stoch-dynamics}\\
        x(0) \sim &\, \mathcal{N}(m_0, \Sigma_0),\, x(1) \sim \mathcal{N}(m_1, \Sigma_1),
    \end{align} 
    where $\omega(t)$ is standard additive Gaussian white noise and $\epsilon$ is the entropy regularisation term used to simplify the optimisation problem \cite{ot_control_2021}. Note that the time interval here $t \in [0,1]$ is arbitrarily scaled to this range. 

	\section{Methodology} \label{sec:depth-dynamics}

    \subsection{Task-Space port-Hamiltonian Dynamics}\label{sec:sub:ph-dynamics}
    
    For creating an energy-based controller for the visual servoing problem, we first define the closed-loop dynamics in a port-based structure. This is done by deriving the Hamiltonian dynamics on the $\mathtt{SE}(3)$ Lie group and then formulating the closed-loop port-Hamiltonian structure for the manipulator. \\
    
    As we are working in the task space of the manipulator but controlling the joint space, the first step is to find a relationship between the joint space and the task space that maintains the geometric qualities. We can determine the body velocity $\boldsymbol{\xi}_b : \mathbb{R} \rightarrow \mathbb{R}^3 \bigoplus \mathbb{R}^3 $ using \cite[Lemma 5.3]{geo_control_2005} as $\boldsymbol{\xi}_b = (\boldsymbol{g}^{-1}\dot{\boldsymbol{g}})^\vee$. As the configuration space for a robotic manipulator is normally denoted as the joints of the robot, we want to obtain a relationship for the forward kinematics map that maintains the geometric representation of the task space. \\
    
    The forward's kinematics map $\boldsymbol{g}: \mathcal{Q} \rightarrow \mathtt{SE}(3)$ can be written in terms of unit twists for each joint \cite{math_intro_manip_1994}:
    \begin{equation}
    \boldsymbol{g}(\mathfrak{q}) = \left ( \prod_{i=1}^n e^{\hat{\boldsymbol{\xi}}_iq_i} \right ) \boldsymbol{g}(0),
    \end{equation}
    
    \noindent where $\boldsymbol{g}(0)$ is the zero configuration of the manipulator and the unit twists for each joint are defined using (\ref{eq:rod}). The body velocity of the end effector can be defined as a twist:
    \begin{equation}\label{eq:body-vec}
    \begin{aligned}
        \boldsymbol{\xi}_b = &\ \mathbf{J}_b(\mathfrak{q})\dot{\mathfrak{q}},\\
    \end{aligned}
    \end{equation}
    where:
    \begin{equation}
        \begin{aligned}
            \mathbf{J}_b(\mathfrak{q}) = & \left [ \boldsymbol{\xi}_1^\dagger,\, \cdots,\, \boldsymbol{\xi}_{n-1}^\dagger,\, \boldsymbol{\xi}_{n}^\dagger  \right ],\\
        \boldsymbol{\xi}_{i}^\dagger = &\, \text{Ad}^{-1}_{\boldsymbol{g}(\mathfrak{q})}\, \boldsymbol{\xi}_i.
        \end{aligned}
    \end{equation}
    \noindent is the body Jacobian, relating the end-effector twists relative to the joint velocities in the tool frame \cite{math_intro_manip_1994}. Using this relationship, we can define the dynamics of the manipulator whilst maintaining the geometry of the end-effectors pose. Let $\mathfrak{g} \triangleq \boldsymbol{g}(\mathfrak{q}) \in \mathtt{SE}(3)$ be the end-effector generalized position. From (\ref{eq:left-vector}), we can define the kinematic constraint for the velocity as
    \begin{equation}\label{eq:lie-state}
    \dot{\mathfrak{g}} = \mathsf{T}_e\mathsf{L}_\mathfrak{g}(\hat{\boldsymbol{\xi}}) = \mathfrak{g}\hat{\boldsymbol{\xi}}.
    \end{equation}
    
    We can reformat the Lagrangian in (\ref{eq:lagrange}) to incorporate the pose and twist velocities:
    \begin{equation}\label{eq:lagrange-twist}
    \mathcal{L}(\mathfrak{g}, \hat{\boldsymbol{\xi}}) = \mathcal{K}(\mathfrak{g}, \hat{\boldsymbol{\xi})} - \mathcal{G}(\mathfrak{g}),
    \end{equation}
    \noindent which is defined in the task-space of the manipulator. To define our dynamics, we must take into consideration the relationship in (\ref{eq:body-vec}) between the joint space and the task space velocities. We can use the joint space of the manipulator to derive task space dynamics that maintain the geometric representation of the end-effector whilst enabling joint-level control of the manipulator \cite{handbook_robotics_2008}. The kinetic energy in (\ref{eq:lagrange-twist}) is found for the end effector using the relationship in (\ref{eq:body-vec}):
    \begin{equation}\label{eq:kinetic}
    \begin{aligned}
        \mathcal{K}(\mathfrak{g}, \hat{\boldsymbol{\xi}}) & = \frac{1}{2}\dot{\mathfrak{q}}^\top\mathbf{M}(\mathfrak{q})\dot{\mathfrak{q}}\\
        & = \frac{1}{2}\hat{\boldsymbol{\xi}}^\top_b\mathbf{J}^{+\top}_b(\mathfrak{q})\mathbf{M}(\mathfrak{q})\mathbf{J}^{+}_b(\mathfrak{q})\hat{\boldsymbol{\xi}}_b\\
        & = \frac{1}{2}\hat{\boldsymbol{\xi}}^\top_b \tilde{\mathbf{M}}(\mathfrak{q})\hat{\boldsymbol{\xi}}_b,
    \end{aligned}
    \end{equation}
    
    \noindent where $\mathbf{M}(\mathfrak{q}) \in \mathbb{S}^{n}_{++}$ is the mass-inertia matrix of the manipulator and $\tilde{\mathbf{M}}(\mathfrak{q}) = \mathbf{J}^{+\top}_b(\mathfrak{q})\mathbf{M}(\mathfrak{q})\mathbf{J}^{+}_b(\mathfrak{q})$, where $\mathbf{J}^+(\mathfrak{q})$ denotes the Moore-Penrose pseudo-inverse of the geometric Jacobian from (\ref{eq:body-vec}). Therefore we define the task-space Lagrangian as:
    \begin{equation}\label{eq:task-lagrange}
    \mathcal{L}(\mathfrak{q}, \hat{\boldsymbol{\xi}}) = \frac{1}{2}\hat{\boldsymbol{\xi}}^\top_b \tilde{\mathbf{M}}(\mathfrak{q})\hat{\boldsymbol{\xi}}_b - \mathcal{G}(\mathfrak{g}).
    \end{equation}
    
    Using the Lagrangian, the task-space Hamiltonian can be found using the Legendre transformation from (\ref{eq:legendre-hamilton}) on the cotangent bundle $(\mathfrak{g}, \mathfrak{p}) \in \mathsf{T}^\ast\mathtt{SE}(3)$:
    \begin{equation}\label{eq:hamiltonian}
    \mathcal{H}(\mathfrak{g}, \mathfrak{p}) = \mathfrak{p}\cdot\hat{\boldsymbol{\xi}} - \mathcal{L}(\mathfrak{q}, \hat{\boldsymbol{\xi}}),
    \end{equation}
    \noindent with the conjugate momenta $\mathfrak{p}$ defined as:
    \begin{equation} \label{eq:momenta}
    \mathfrak{p} = \left [\begin{array}{c}
        \mathfrak{p}_{\boldsymbol{v}}\\
        \mathfrak{p}_{\boldsymbol{\omega}}
    \end{array}\right ] = \frac{\partial\mathcal{L}(\mathfrak{g}, \hat{\boldsymbol{\xi}})}{\partial\hat{\boldsymbol{\xi}}} = \tilde{\mathbf{M}}(\mathfrak{q})\hat{\boldsymbol{\xi}}^\vee = \tilde{\mathbf{M}}(\mathfrak{q})\mathbf{J}_b(\mathfrak{q})\dot{\mathfrak{q}}.
    \end{equation}
    
    To find the closed-form Hamiltonian equations for the manipulator, we leverage the geometry of the $\mathtt{SE}(3)$ group to find the way the state $(\mathfrak{g}, \mathfrak{q}) \in \mathsf{T}^\ast\mathtt{SE}(3)$ evolves in time \cite{global_hamilton_2018}. By combining (\ref{eq:legendre-hamilton}), (\ref{eq:task-lagrange}) and (\ref{eq:momenta}), the manipulator Hamiltonian becomes:
    \begin{equation}\label{eq:hamil-manip}
        \begin{aligned}
            \mathcal{H}(\mathfrak{g}, \mathfrak{p}) & = \mathfrak{p}\cdot\hat{\boldsymbol{\xi}} - \left (\frac{1}{2}\hat{\boldsymbol{\xi}}^\top_b \tilde{\mathbf{M}}(\mathfrak{q})\hat{\boldsymbol{\xi}}_b - \mathcal{G}(\mathfrak{g})\right )\\
            & = \mathfrak{p}^{\top}\tilde{\mathbf{M}}^{-1}(\mathfrak{q})\mathfrak{p} - \frac{1}{2}\mathfrak{p}^\top\tilde{\mathbf{M}}^{-1}(\mathfrak{q})\tilde{\mathbf{M}}(\mathfrak{q})\tilde{\mathbf{M}}^{-1}(\mathfrak{q})\mathfrak{p} + \mathcal{G}(\mathfrak{g})\\
            & = \frac{1}{2}\mathfrak{p}^\top\tilde{\mathbf{M}}^{-1}(\mathfrak{q})\mathfrak{p} + \mathcal{G}(\mathfrak{g}).
        \end{aligned}
    \end{equation}
    From (\ref{eq:lie-state}) and (\ref{eq:hamil-manip}), we can find the Hamiltonian dynamics for the task-space Lie group \cite{global_hamilton_2018}:
    \begin{align}
    \dot{\mathfrak{g}} = &\, \mathsf{T}_e\mathsf{L}_{\mathfrak{g}}\left (\frac{\partial\mathcal{H}(\mathfrak{g}, \mathfrak{p})}{\partial\mathfrak{p}}\right) = \mathsf{T}_e\mathsf{L}_{\mathfrak{g}} \left( \tilde{\mathbf{M}}^{-1} \begin{bmatrix}
        \mathfrak{p}_{\boldsymbol{v}}\\
        \mathfrak{p}_{\boldsymbol{\omega}}
    \end{bmatrix}\right), \label{eq:hamil-dyn-g}\\
    \dot{\mathfrak{p}} = &\, \mathtt{ad}^*_{\hat{\boldsymbol{\xi}}}(\mathfrak{p}) - \mathtt{T}^\ast_e\mathtt{L}_\mathfrak{g} \left ( \frac{\partial \mathcal{H}(\mathfrak{g}, \mathfrak{p})}{\partial \mathfrak{g}} \right ) + \mathbf{B}(\mathfrak{q})\mathbf{u}. \label{eq:hamil-dyn-p} 
    \end{align}
    To obtain the dynamics in port-Hamiltonian form as shown in Definition \ref{def:port-hamil}, the tasks-space coordinates can be vectorised as $\mathfrak{g} = [\mathbf{p}^\top, \mathbf{r}_1^\top, \mathbf{r}_2^\top, \mathbf{r}_3^\top]^\top$, and the conjugate momenta can be split into translational and rotational velocity components from (\ref{eq:momenta}) such that we can define the flattened state vector as:
    \begin{equation}\label{eq:flatten-vec}
    \mathbf{x} =  \begin{bmatrix}
        \mathfrak{g} \\
        \mathfrak{p}
        \end{bmatrix} = \begin{bmatrix} \mathbf{p}^\top,\mathbf{r}_1^\top, \mathbf{r}_2^\top, \mathbf{r}_3^\top, \mathfrak{p}_v, \mathfrak{p}_\omega \end{bmatrix}^\top.
    \end{equation}
    The state dynamics can therefore be found as \cite{global_hamilton_2018}:
    \begin{align}
        \dot{\mathbf{p}} & = \mathbf{R}\frac{\partial\mathcal{H}}{\partial\mathfrak{p}_v}, \\
        \dot{\mathbf{r}}_i & = \mathbf{r}_i \times \frac{\partial\mathcal{H}}{\partial\mathfrak{p}_\omega} \ \text{for} \ i=1,2,3, \\
        \dot{\mathfrak{p}}_v & = \mathfrak{p}_v \times \frac{\partial\mathcal{H}}{\partial\mathfrak{p}_\omega} - \mathbf{R}^\top\frac{\partial\mathcal{H}}{\partial\mathbf{p}} + \mathbf{b}_v(\mathfrak{q})\mathbf{u}, \\
        \dot{\mathfrak{p}}_\omega & = \mathfrak{p}_\omega\times\frac{\partial\mathcal{H}}{\partial\mathfrak{p}_\omega} + \mathfrak{p}_v\times\frac{\partial\mathcal{H}}{\partial\mathfrak{p}_v} + \left [ \sum^3_{i=1} \mathbf{r}_i\times \frac{\partial\mathcal{H}}{\partial\mathbf{r}_i} \right ] + \mathbf{b}_\omega(\mathfrak{q})\mathbf{u}.
    \end{align}
    \begin{proposition}\label{prop:energy-conservation}
        When $\mathbf{u} = \mathbf{0}$ (zero control input to the system), the Hamiltonian in (\ref{eq:hamil-manip}) ensures the conservation of energy, such that $\frac{\text{d}\mathcal{H}}{\text{d}t} = 0$.
    \end{proposition}
    \begin{proof}
        We provide the proof of this result in Appendix \ref{app:energy-conservation}.
    \end{proof}
    
    With $\dot{\mathbf{x}} = \begin{bmatrix}
        \dot{\mathfrak{g}}\ & \dot{\mathfrak{q}}
    \end{bmatrix}^\top$, the port-Hamiltonian dynamics in Definition \ref{def:port-hamil} can be defined as:
    \begin{equation}\label{eq:pH-manip}
    \dot{\mathbf{x}} = \underbrace{\left [ \begin{array}{cc}
        0 & \mathfrak{g}^{\times}\\
        -\mathfrak{g}^{\times\top} & \mathfrak{p}^{\times}
    \end{array} \right ]}_{\mathcal{J}(\mathfrak{g},\mathfrak{p})}\nabla\mathcal{H}(\mathfrak{g}, \mathfrak{p}) + \begin{bmatrix}
        0 \\
        \mathbf{B}(\mathfrak{q})
    \end{bmatrix}\mathbf{u},
    \end{equation}
    where 
    \begin{equation}
        \mathfrak{g}^\times = \left[ \begin{array}{cccc}
    \mathbf{R}^\top\, & 0 & 0 & 0 \\
    0 & \hat{\mathbf{r}}_1^\top & \hat{\mathbf{r}}_2^\top & \hat{\mathbf{r}}_3^\top
    \end{array} \right]^\top, \ \mathfrak{p}^\times = \left[\begin{array}{cc}
    0 & \hat{\mathfrak{p}_v} \\
    \hat{\mathfrak{p}}_v & \hat{\mathfrak{p}}_\omega
    \end{array}\right]
    \end{equation}
    provide us with the skew-symmetric interconnection matrix $\mathcal{J}(\mathfrak{g}, \mathfrak{p})$ \cite{schaft_port_2004,duong_hamiltonian_2021}. This equation bears similarity in form to the stochastic dynamics in (\ref{eq:sub:stoch-dynamics}), which allows the formulation of the port-Hamiltonian dynamics as a stochastic control problem \cite{duong_csl_2022}. These dynamics are generalised such that they can be used in the control design for any \textit{fully-actuated} multi-body system, such as continuum-style robots \cite{continuum_dynamics_2024} and vehicle suspension \cite{vehicle_dynamics_1999}, provided that there exists a forward's kinematic map $\boldsymbol{g}(\mathfrak{q})$ allowing the computation of the end-effector pose in $\mathtt{SE}(3)$.
    
    \begin{figure*}[t]
    	\centering
    	\includegraphics[width=0.6\textwidth]{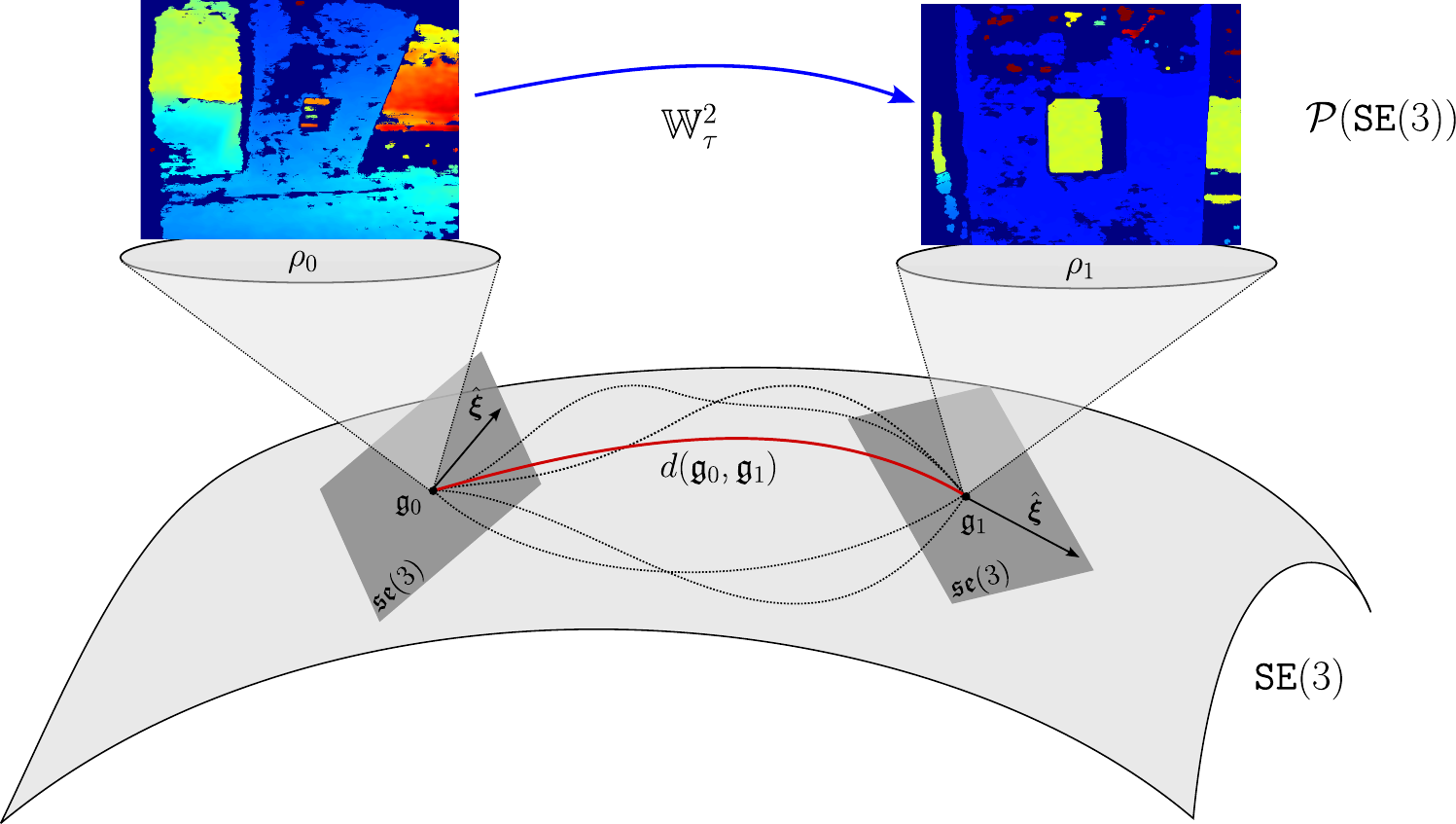}
    	\caption{Example of the depth maps supported on the $\mathtt{SE}(3)$ Lie group. The initial depth map $\rho_0 \in \mathcal{P}(\mathtt{SE}(3))$ is supported on the initial pose $\mathfrak{g}_0 \in \mathtt{SE}(3)$, and the target depth map $\rho_1 \in \mathcal{P}(\mathtt{SE}(3))$ is supported on the target pose $\mathfrak{g}_1 \in \mathtt{SE}(3)$. The geodesic distance $d(\mathfrak{g}_0, \mathfrak{g}_1)$ maps the shortest distance in the group space whilst the $p$-Wasserstein distance $\mathbb{W}_\tau^2$ maps the initial to target depth maps in probability space. The colours in the depth maps correspond to the distance from the camera, with blue particles indicating close-range measurements and green particles indicating the distance measurements within the limit of the camera range.}
    	\label{fig:curvature-example}
    \end{figure*}
    
    \subsection{Transport Maps of Depth Clouds}\label{sec:sub:features}
    
    Modern approaches to visual servoing have utilised better depth estimation and reliable feature extraction of point clouds over traditional 2-dimensional cameras \cite{2d_vs_1999,vs_point_cloud_2020,deep_vs_2022}. Depth maps lend themselves well to the mass transport problem, as they can be represented as distributions on probability spaces \cite{ot_change_detect_2023,ot_metric_point_2024}. \\
    
    The most common camera type for deployed systems is the RGB-D camera, which utilises stereo vision to compute distance estimates and creates a dense depth map of similar accuracy and density to that of more expensive LiDAR systems \cite[Chap. 40.5.3]{handbook_robotics_2008}. A depth cloud $P$ is a finite set of points, where each point relates to a camera pixel in the resolution range $[\mathbb{X}^n, \mathbb{Y}^m]$ and the distance estimate:
    \begin{equation}
        P = \left \{ [x, y, Z] \in \mathbb{R}^3_{+}\, \lvert\, x \in \mathbb{X}^n, y \in \mathbb{Y}^m, Z \sim f(x, y) \right \},
    \end{equation}
    where $n$ and $m$ are the resolution bounds for the camera in the form $n\times m$. For both the initial $P_\mu$ and goal pose $P_\nu$, the point cloud of the pose is the set of features that are present in the pose of the camera. To aid with notational compactness, for this section of the report we denote $\mathtt{G}$ as the Special Euclidean group $\mathtt{SE}(3)$ mentioned previously. \\
    
    At any given point $\mathfrak{g} \in \mathtt{TG}$ of the robot task space, the depth cloud from the camera provides us with the reading $P_\mu$. The target pose $\mathfrak{g}^\ast \in \mathtt{TG}$ possesses a target depth cloud $P_\nu$, where both of the depth clouds can be shown as discrete measures for each point in the depth cloud: 
    \begin{equation} \label{eq:dirac-measures}
        \mu = \sum_{i=1}^{n}\sum_{j=1}^{m} \mathbf{u}_{ij} \delta_{x_i y_j}, \quad \nu = \sum_{i=1}^{n} \sum_{j=1}^{m} \mathbf{v}_{ij} \delta_{x_i y_j},
    \end{equation}
    where $(x_i, y_j)$ corresponds to pixel locations in the frame, and $\delta$ is the Kronecker delta function. As the depth values used in the depth cloud are signed positive, we say they exist on the space of finite signed measures $\mathcal{M}(\mathtt{G})$, forming a positive cone $\mathcal{M}_{+}(\mathtt{G}) = \left\{ \mu \in \mathcal{M}(\mathtt{G}) : \mu \geq 0 \right\}$, giving a probability measure space of:
    \begin{equation}
        \mathcal{P}(\mathtt{G}) = \left\{ \mu \in \mathcal{M}_+(\mathtt{G}) : \mu(\mathtt{G}) = 1 \right\}.
    \end{equation}
    These measures exist on the Wasserstein space $\mathcal{P}_p({\mathtt{G}})$, which consists of the set of all probability measures $\mu \in \mathcal{P}(\mathtt{G})$:
    \begin{equation}
        \mathcal{P}_p(\mathtt{G}) = \left \{ \mu \in \mathcal{P}(\mathtt{G}) : \int_{\mathtt{G}} d^p(\mathfrak{g}_0, \mathfrak{g}_1)\;\text{d}\mu(x) < +\infty \right \}, \quad \forall p \geq 1,
    \end{equation}
    with $d(\mathfrak{g}_0, \mathfrak{g}_1)$ defining the geodesic distance between two points $\mathfrak{g}_0, \mathfrak{g}_1 \in \mathtt{G}$:
    \begin{equation}
        d(\mathfrak{g}_0, \mathfrak{g}_1) \triangleq \inf_{\gamma \in PC([0,1])} \int_{0}^{1} \sqrt{\dot{\gamma}^\top(t)\mathbb{G}_{\gamma(t)}\dot{\gamma}(t)}\; \text{d}t 
    \end{equation}
    where $\gamma(0) = \mathfrak{g}_0$ and $\gamma(1) = \mathfrak{g}_1$. Our choice of metric tensor $\mathbb{G}_{\gamma(t)}$ and computation of the geodesic distance $d(\mathfrak{g}_0, \mathfrak{g}_1)$ is detailed in Appendix \ref{app:metric-tensor}. The $p$-Wasserstein distance from (\ref{eq:wasserstein}) can be formulated for the two depth cloud measures $(\mu, \nu) \in \mathcal{P}_p(\mathtt{G})$:
    \begin{equation}\label{eq:wasserstein-space}
        \mathbb{W}^p(\mu, \nu) \triangleq \left( \inf_{\pi \in \Pi(\mathtt{G}\times \mathtt{G}; \mu,\nu)} \int_{\mathtt{G}\times \mathtt{G}} d^p(x,y) \text{d}\pi(x,y)  \right)^{\frac{1}{p}}.
    \end{equation}
    This distance, when applied to the $\mathcal{P}_p(\mathtt{G})$ space, describes the $L^p(\mathtt{G})$-Wasserstein space, where the distance is an optimal transport problem provided that the cost function $c(x,y) = d^p(x, y)$ is lower semi-continuous \cite{ot_lie_2024}. \\
    
    Most optimal transport problems can be solved with the assumption that the total mass is conserved during operation (i.e. $\int_{\mathtt{G}}\text{d}\mu = \int_{\mathtt{G}}\text{d}\nu$). However, this assumption is not necessarily true when working with point clouds as distances are bounded from below (i.e. $\mu \nleq 0$) but aren't bounded from above\footnote{Theoretically, one could provide an upper bound and maintain equivalent mass transportation based on the maximum range of the camera under the assumption that the camera starts at this distance. We leave this as an avenue for future research.}, which leads to $\Pi(\mu, \nu) = \emptyset$ and $\mathbb{W}^p(\mu, \nu) = +\infty$. To avoid this problem, we modify (\ref{eq:wasserstein}) to use the unbalanced Kantorovich relaxation \cite{ot_book_2015}:
    \begin{equation}\label{eq:unbalanced-wasserstein}
        \begin{aligned}
            \mathbb{W}_\tau^p(\mu, \nu) = & \bigg( \inf_{\pi \in \Pi(\mathtt{G}\times \mathtt{G}; \mu,\nu)} \int_{\mathtt{G}\times \mathtt{G}} d^p(x,y) \text{d}\pi(x,y) + \tau\mathbb{D}_\varphi(P_{1,\sharp}\pi | \mu) + \tau\mathbb{D}_\varphi(P_{2,\sharp}\pi|\nu)\bigg)^{\frac{1}{p}}.
        \end{aligned}
    \end{equation}
    When the scaling value $\tau \rightarrow +\infty$, assuming that the depth maps have equal density such that $\mu = \nu$, we recover the original definition for the $p$-Wasserstein distance \cite{comp_ot_2019}. Figure \ref{fig:curvature-example} provides an example of the target and initial depth maps supported on $\mathtt{SE}(3)$. 
    
    \subsection{Passivity-Based Control on $\mathtt{SE}(3)$}\label{sec:sub:closed-loop}
    
    The port-Hamiltonian dynamics in (\ref{eq:pH-manip}) describe an open-loop system that has a minimum energy derived from to (\ref{eq:hamil-manip}). The underlying principle with visual servoing - in particular PBVS - is to control the system to a desired stabilisation point $(\mathfrak{g}^\ast, \mathfrak{p}^\ast) \in \mathtt{T^\ast SE}(3)$. However, there is no guarantee that there exists an energy minimum in $\mathcal{H}(\mathfrak{g}, \mathfrak{p})$ at the desired stabilisation point. Furthermore, from Proposition \ref{prop:energy-conservation}, the passivity of the system dictates that the outputs are shaped by the generalised velocities only. As such, it is beneficial to reshape the Hamiltonian dynamics to allow the injection of energy to reach the desired total energy $\mathcal{H}_d(\mathfrak{g}, \mathfrak{p})$ \cite{ph_stabil_2002}. This gives a modified Hamiltonian of:
    \begin{equation}\label{eq:hamil-inject}
        \mathcal{H}_d(\mathfrak{g}, \mathfrak{p}) = \mathcal{H}(\mathfrak{g}, \mathfrak{p}) + \mathcal{H}_a(\mathfrak{g}, \mathfrak{p}),
    \end{equation}
    where $\mathcal{H}_a(\mathfrak{g}, \mathfrak{p})$ is the energy injection to the system to reach the desired equilibrium point. From Proposition \ref{prop:energy-conservation}, as our system is passive we utilise a \textit{passivity-based controller} (PBC) \cite{ph_stabil_2002,schaft_port_2004} to inject energy into our system. The control input then seeks to provide an optimal input that minimises $\mathcal{H}_d(\mathfrak{g}, \mathfrak{p})$ as the desired equilibrium:
    \begin{equation}
        (\mathfrak{g}^\ast, \mathfrak{p}^*) = \argmin_{(\mathfrak{g}, \mathfrak{p}) \in \mathtt{T}^\ast\mathtt{SE}(3)} \big\{ \mathcal{H}_d(\mathfrak{g}, \mathfrak{p})\big\},
    \end{equation}
    with the port-Hamiltonain dynamics of $\mathcal{H}_d$ being:
    \begin{equation}\label{eq:target-ph}
        \begin{bmatrix}
            \dot{\mathfrak{g}} \\
            \dot{\mathfrak{p}}
        \end{bmatrix} = \left[\mathcal{J}_d(\mathfrak{g}, \mathfrak{p}) - \mathcal{R}_d(\mathfrak{g}, \mathfrak{p})\right] \nabla\mathcal{H}_d(\mathfrak{g}, \mathfrak{p}).
    \end{equation}
    By equating (\ref{eq:pH-manip}) and (\ref{eq:target-ph}), we obtain the following control requirement:
    \begin{equation}\label{eq:control-req}
        \begin{aligned}
            \mathbf{u} = \mathbf{B}^{+}(\mathfrak{q})\big([\mathcal{J}_d(\mathfrak{g}, \mathfrak{p}) - &\, \mathcal{R}_d(\mathfrak{g}, \mathfrak{p})]\nabla\mathcal{H}_d(\mathfrak{g}, \mathfrak{p}) - \mathcal{J}(\mathfrak{g}, \mathfrak{p})\nabla\mathcal{H}(\mathfrak{g}, \mathfrak{p}))\big),
        \end{aligned}
    \end{equation}
    where $\mathbf{B}^{+}(\mathfrak{q}) = (\mathbf{B}^\top\mathbf{B})^{-1}\mathbf{B}^\top$ denotes the Moore-Penrose pseudo-inverse of the actuation matrix. The desired dissipation matrix is defined using the damping injection matrix $\mathbf{K}_d$ as $\mathcal{R}_d(\mathfrak{g}, \mathfrak{p}) = \mathbf{B}(\mathfrak{q})\mathbf{K}_d\mathbf{B}^\top(\mathfrak{q})$. The state feedback control is defined as the sum of the energy-shaping and damping injection terms $\mathbf{u} = \mathbf{u}_{\text{ES}} + \mathbf{u}_{\text{DI}}$. These are determined from (\ref{eq:control-req}) as:
    \begin{align}
        \mathbf{u}_{\text{ES}} = &\, \mathbf{B}^{+}(\mathfrak{q})\left[\mathcal{J}_d(\mathfrak{g}, \mathfrak{p})\nabla\mathcal{H}_d(\mathfrak{g}, \mathfrak{p}) - \mathcal{J}(\mathfrak{g}, \mathfrak{p})\nabla\mathcal{H}(\mathfrak{g}, \mathfrak{p})\right],\label{eq:energy-gain} \\
        \mathbf{u}_{\text{DI}} = &\, -\mathbf{K}_d\mathbf{B}^\top(\mathfrak{q})\nabla\mathcal{H}_d(\mathfrak{g}, \mathfrak{p}).\label{eq:damping-gain}
    \end{align}
    
    For the design of the closed-loop dynamics, the options for computing stable control gains are based on fixing either the $\mathcal{J}_d$ and $\mathcal{R}_d$ matrices or $\mathcal{H}_d$ \cite{port_hamilton_overview_2014}. As finding matrices remains an open problem in port-Hamilton control theory, we instead select matrices that satisfy the matching condition
    \begin{equation}\label{eq:matching-condition}
        \begin{aligned}
            \mathbf{B}^\bot \bigg( [ (\mathcal{J}(\mathfrak{g}, \mathfrak{p})& - \mathcal{R}(\mathfrak{g}, \mathfrak{p}))\nabla\mathcal{H}(\mathfrak{g}, \mathfrak{p}) ] - \left[ (\mathcal{J}_d(\mathfrak{g}, \mathfrak{p}) - \mathcal{R}_d(\mathfrak{g}, \mathfrak{p}))\nabla\mathcal{H}_d(\mathfrak{g}, \mathfrak{p}) \right] \bigg) = 0,
        \end{aligned}
    \end{equation}
    where $\mathcal{B}^\bot$ is the maximal-rank left annihilator of $\mathcal{B}$ such that $\mathcal{B}^\bot\mathcal{B}=0$. The target phase-space coordinates $(\mathfrak{g}^\ast, \mathfrak{p}^\ast)$ can be flattened in a similar manner to (\ref{eq:flatten-vec}), which allows $\mathcal{J}_d$ to be formulated as:
    \begin{equation}\label{eq:desired-interconnection}
        \begin{aligned}
            \mathcal{J}_d(\mathfrak{g},\mathfrak{p})& \, = \begin{bmatrix}
                \mathbf{0} & \mathbf{J}_1 \\
                -\mathbf{J}_1^\top & \mathbf{0}
            \end{bmatrix}, \\
            \text{where } \mathbf{J_1} = & \, \begin{bmatrix}
                \mathbf{R}^\top & 0 & 0 & 0 \\
                0 & \hat{\mathbf{r}}^{\ast\top}_{1} & \hat{\mathbf{r}}^{\ast\top}_{2} & \hat{\mathbf{r}}^{\ast\top}_{3}
            \end{bmatrix}^\top.
        \end{aligned}
    \end{equation}
    We can similarly choose $\mathcal{R}_d(\mathfrak{g},\mathfrak{p}) = \begin{bmatrix}
        0 & 0 \\
        0 & \mathbf{K}_d
    \end{bmatrix}$, which satisfies both $\mathcal{J}_d = -\mathcal{J}_d^\top$ and $\mathcal{R}_d = \mathcal{R}_d^\top \succeq 0$ whilst including the damping gain $\mathbf{K}_d$ \cite{duong_hamiltonian_2021}. For our energy injection term $\mathcal{H}_a$, the simplest way to guarantee that the minimum energy exists at $(\mathfrak{g}^\ast, \mathfrak{p}^\ast) = (\mathfrak{g}^\ast, 0)$ is to use energy shaping gains $\mathbf{K}_\mathbf{p}, \mathbf{K}_\mathbf{R} \succ 0$:
    \begin{equation}\label{eq:ha-shaping}
        \begin{aligned}
            \mathcal{H}_a(\mathfrak{g}, \mathfrak{p}) = -\mathcal{H}(\mathfrak{g}, \mathfrak{p}) + \frac{1}{2}(\mathfrak{p}-\mathfrak{p}^\ast)^\top\tilde{\mathbf{M}}^{-1}(\mathfrak{q})(\mathfrak{p}-\mathfrak{p}^\ast) + \frac{1}{2}(\mathbf{p}-\mathbf{p}^\ast)^\top\mathbf{K}_\mathbf{p}(\mathbf{p}-\mathbf{p}^\ast) + \frac{1}{2}tr(\mathbf{K}_\mathbf{R}(\mathbf{I}-\mathbf{R}^{\ast\top}\mathbf{R})).
        \end{aligned}
    \end{equation}
    By assuming that $\mathcal{G}_d(\mathfrak{g}^\ast) = \mathcal{G}(\mathfrak{g}^\ast)$, the desired Hamiltonian becomes:
    \begin{equation}\label{eq:hd-desired}
        \begin{aligned}
            \mathcal{H}_d(\mathfrak{g}, \mathfrak{p})& = \frac{1}{2}(\mathfrak{p}-\mathfrak{p}^\ast)^\top\tilde{\mathbf{M}}^{-1}(\mathfrak{q})(\mathfrak{p}-\mathfrak{p}^\ast) + \frac{1}{2}(\mathbf{p}-\mathbf{p}^\ast)^\top\mathbf{K}_\mathbf{p}(\mathbf{p}-\mathbf{p}^\ast) + \frac{1}{2}\text{tr}(\mathbf{K}_\mathbf{R}(\mathbf{I}-\mathbf{R}^{\ast\top}\mathbf{R}))
        \end{aligned}
    \end{equation}
    Substituting (\ref{eq:hd-desired}) into (\ref{eq:control-req}), (\ref{eq:energy-gain}) and (\ref{eq:damping-gain}), it follows that:
    \begin{align}
        \mathbf{u}_{\text{ES}}&\; = \mathbf{B}^+\left[\mathfrak{g}^{\times\top}\frac{\partial\mathcal{H}}{\partial\mathfrak{g}} - \mathfrak{p}^{\times}\boldsymbol{\xi}- \mathbf{J}_1^\top \frac{\partial\mathcal{H}_d}{\partial\mathfrak{g}} \right],\label{eq:es-final}\\
        \mathbf{u}_{\text{DI}} & = -\mathbf{K}_d\mathbf{B}^\top\boldsymbol{\xi}.\label{eq:di-final}
    \end{align}
    We can further simplify (\ref{eq:es-final}) as:
    \begin{equation}
        \mathbf{u}_{\text{ES}} = \mathbf{B}^+(\mathfrak{q})\left[ \mathfrak{g}^{\times\top}\nabla_{\mathfrak{g}}\mathcal{G}(\mathfrak{g}) - \mathfrak{p}^\times\boldsymbol{\xi} - \text{e}(\mathfrak{g}, \mathfrak{g}^\ast) \right]  ,
    \end{equation}
    where:
    \begin{equation}
        \text{e}(\mathfrak{g}, \mathfrak{g}^\ast) \triangleq \mathbf{J}_1^\top\frac{\partial\mathcal{H}_d}{\partial\mathfrak{g}} = \begin{bmatrix}
            \mathbf{R}^\top\mathbf{K}_{\mathbf{p}}(\mathbf{p}-\mathbf{p}^\ast)\\
            \frac{1}{2}(\mathbf{K}_\mathbf{R}\mathbf{R}^{\ast\top}\mathbf{R} - \mathbf{R}^\top\mathbf{R}^\ast\mathbf{K}_{\mathbf{R}}^\top)^\vee
        \end{bmatrix}
    \end{equation}
    denotes the geometric error between the desired and actual poses. Note that $\nabla_\mathfrak{g} \mathcal{G}(\mathfrak{g}) = \mathbf{J}_b^{-\intercal}\mathcal{G}(\mathfrak{q})$ is the gravity vector defined in the $\mathtt{SE}(3)$ task space \cite{contact_rich_2024}. The control framework here can be thought of as proportional-derivative control with gravity-compensation \cite{port_hamilton_overview_2014}, where the matrices $\mathbf{K}_\mathbf{p}, \mathbf{K}_\mathbf{R}, \mathbf{K}_d \succ 0$ are chosen to exhibit the desired performance. 
    \begin{theorem} \label{thm:cl-stability}
        Consider the port-Hamiltonian dynamics for the multi-body system in (\ref{eq:pH-manip}). Given that the matching condition (\ref{eq:matching-condition}) is satisfied and the gain matrices $\mathbf{K}_\mathbf{p}, \mathbf{K}_\mathbf{R}$ and $\mathbf{K}_d$ are positive definite, for the chosen time-invariant Lyapunov function in (\ref{eq:hd-desired}), the system exhibits closed-loop asymptotic stability for all $(\mathfrak{g}, \mathfrak{p}) \in \mathtt{T}^\ast\mathtt{SE}(3)$.
    \end{theorem}
    \begin{proof}\label{proof:lyapunov-stability}
        Let us assume that the manipulator body Jacobian is of full rank:
        \begin{equation}
        	\text{rank}(\mathbf{J}_\mathfrak{b}(\mathfrak{q})) = \max (\mathcal{Q}).
        \end{equation}
        Using (\ref{eq:desired-interconnection}) and the damping gain $\mathbf{K}_d$, we can formulate the closed loop dynamics for the manipulator as:
        \begin{equation}
            \begin{bmatrix}
                \dot{\mathfrak{g}}\\
                \dot{\mathfrak{p}}
            \end{bmatrix} = \begin{bmatrix}
                \mathbf{0} & \mathbf{J}_1 \\
                -\mathbf{J}_1^\top & -\mathbf{K}_d
            \end{bmatrix}\nabla\mathcal{H}_d.
        \end{equation} 
        Let us define the error between the desired and current dynamics given by
        \begin{equation}
            (\mathfrak{g}_e, \mathfrak{p}_e) = ((\mathbf{p}_e, \mathbf{R}_e), \mathfrak{p}_e) = \begin{bmatrix}
            \mathbf{p}-\mathbf{p}^\ast, \mathbf{R}^{\ast\top}\mathbf{R}, \mathfrak{p}-\mathfrak{p}^\ast
        \end{bmatrix}^\top.
        \end{equation} 
        This allows us to reformulate the Lyapunov function (\ref{eq:hd-desired}) as:
        \begin{equation}\label{eq:hd-error}
            \begin{aligned}
                \mathcal{H}_d(\mathfrak{g}_e, \mathfrak{p}_e) =& \frac{1}{2}\mathfrak{p}_e^{\top}\tilde{\mathbf{M}}^{-1}(\mathfrak{q})\mathfrak{p}_e + \frac{1}{2}\mathbf{p}_e^\top\mathbf{K}_\mathbf{p}\mathbf{p}_e + \frac{1}{2}\text{tr}(\mathbf{K}_\mathbf{R}(\mathbf{I} - \mathbf{R}_e)).
            \end{aligned}
        \end{equation}
        As all the values of $\mathbf{R}_e \in \mathtt{SO}(3)$ are less than 1 \cite{duong_hamiltonian_2021}, it holds that:
        \begin{equation}
            \frac{1}{2}\text{tr}(\mathbf{K}_\mathbf{R}(\mathbf{I} - \mathbf{R}_e)) \geq 0,
        \end{equation}
        and that $\mathbf{K}_\mathbf{p}, \mathbf{K}_\mathbf{R}, \mathbf{K}_d \succ 0$. This implies that the Hamiltonian Lyapunov equation $\mathcal{H}_d$ is positive definite, with a minimum value $\mathcal{H}_d = 0$ only when there are no position, rotation, and momentum errors. The time derivative of (\ref{eq:hd-desired}) can be found through the Poisson bracket from Definition \ref{def:lie-poisson-bracket}:
        \begin{equation}
            \begin{aligned}
                \dot{\mathcal{H}}_d(\mathfrak{g}, \mathfrak{p})& = \left\{ \mathcal{H}_d, \mathcal{H} \right\} \\
                & = \left\langle \frac{\partial\mathcal{H}_d}{\partial\mathfrak{g}}, \dot{\mathfrak{g}} \right\rangle + \left\langle \frac{\partial\mathcal{H}_d}{\partial\mathfrak{p}}, \dot{\mathfrak{p}} \right\rangle \\
                & = - (\mathfrak{p}-\mathfrak{p}^\ast)^\top\tilde{\mathbf{M}}^{-1}(\mathfrak{q})\mathbf{K}_d\tilde{\mathbf{M}}^{-1}(\mathfrak{q})(\mathfrak{p}-\mathfrak{p}^\ast).
            \end{aligned}
        \end{equation}
        By using (\ref{eq:hd-error}), it holds that
        \begin{equation}
            \dot{\mathcal{H}}_d(\mathfrak{g}_e, \mathfrak{p}_e) = - \mathfrak{p}_e^\top\tilde{\mathbf{M}}^{-1}(\mathfrak{q})\mathbf{K}_d\tilde{\mathbf{M}}^{-1}\mathfrak{p}_e.
        \end{equation}
        As both $\mathbf{K}_d$ and $\tilde{\mathbf{M}}$ are positive-definite, the Lyapunov derivative is negative-definite $\dot{\mathcal{H}}_d \leq 0$ for all $(\mathfrak{g}, \mathfrak{p}) \in \mathtt{T}^\ast\mathtt{SE}(3)$. As such, by LaSalle's invariance principle \cite{nonlinear_sys_1999}, the Hamiltonian Lyapunov error function converges to 0 and such reaches the desired equilibrium point $(\mathfrak{g}^\ast, \mathfrak{p}^\ast) \in \mathtt{T}^\ast\mathtt{SE}(3)$. 
    \end{proof}
    
    As a result of Theorem \ref{thm:cl-stability}, the Lyapunov function provides asymptotic stability regardless of the choice of $\mathbf{K}_\mathbf{p}, \mathbf{K}_\mathbf{R}$ and $\mathbf{K}_d$. We now discuss how we use the visual features from the depth map to generate a control signal that injects energy into the closed-loop Hamiltonian and mitigates disturbances. 
    
    \subsection{Injecting Energy via Optimal Transport} \label{sec:sub:energy-ot}
    
    So far our controller mimics the standard approach for geometric PD-based control for port-Hamiltonian systems described in  \cite{duong_hamiltonian_2021} and \cite{port_hamilton_overview_2014}, bringing our control process into the PBVS domain on geometric dynamics. However, this control law doesn't account for features extracted from the depth map, leading us back to the optimal transport problem described in Section \ref{sec:sub:features}. Following work in \cite{duong_csl_2022}, the control input $\mathbf{u}$ can be reformulated to include a disturbance compensation term $\mathbf{u}_{\text{DC}}$:
    \begin{equation} \label{eq:vs-pid}
        \mathbf{u} = \mathbf{u}_{\text{ES}} + \mathbf{u}_{\text{DI}} + \mathbf{u}_{\text{DC}}.
    \end{equation}
    The control input $\mathbf{u}_{\text{DC}}$ is dependent on the current measure $\rho_t$ from the RGB camera, which can be formulated as a Hamilton-Jacobi continuity equation \cite{ot_book_2015}:
    \begin{equation}
    	\frac{\partial\rho_t}{\partial t} + \nabla(\rho_t\boldsymbol{\xi}) = s_t,
    \end{equation}
    where $s_t$ is the source term due to the unbalanced measures and $\rho_t\boldsymbol{\xi}$ is the mass momentum of the measure under the influence of the task-space twist $\boldsymbol{\xi}$. This produces a set of curves:
    \begin{equation}
    	\begin{aligned}
    		\bar{\mathcal{C}}(\rho_{0}, \rho_{1}) = \bigg\{ (\rho_t, (\rho_t\boldsymbol{\xi}), s_t) : \frac{\partial\rho_t}{\partial t} + \nabla(\rho_t\boldsymbol{\xi}) = s_t,
    		\rho_{t=0} = \mu, \rho_{t=1}=\nu \bigg\}.
    	\end{aligned}
    \end{equation} 
    The $p$-Wasserstein for these curves derives the mass flow due to the manipulator:
    \begin{equation}
    	\mathbb{W}^2_{\tau} = \min_{(\rho_t, (\rho_t\boldsymbol{\xi}), s_t) \in \bar{\mathcal{C}}(\rho_0, \rho_1)} \int_{0}^{1} \int_{\mathtt{G}} d^p(\mathfrak{g}, \mathfrak{g}_1) + \tau\Theta(\rho_t, s_t) \text{d}\mathfrak{g}\,\text{d}t,
    \end{equation}
    which is the dynamic equivalent to the Kantorovich relaxation of (\ref{eq:unbalanced-wasserstein}) and can be used to compute both the total Wasserstein geodesic from the initial to terminal conditions or the instantaneous geodesic at pose $\mathfrak{g} \in \mathtt{SE}(3)$ to the desired pose. \\
    
    To find the value of $\mathbf{u}_{\text{DC}}$, we leverage the result from \cite{ot_linear_2017} to compute control gains based on the optimal transport map $T^\star$. Let $\Psi(\lambda) = e^{A\lambda}$ be the state transition function for the port-Hamiltonian dynamics described in (\ref{eq:pH-manip}), where $A = \mathcal{J}(\mathfrak{g}, \mathfrak{p})$. The controllability Gramian is denoted as:
    \begin{equation}
    	W_c(t_0, t_1) = \int_{t_0}^{t_1} \Psi(\lambda) \mathbf{B}\mathbf{B}^\top \Psi^{-1}(\lambda) \text{d}\lambda.
    \end{equation}\label{eq:udc_control}
    The control input is defined as:
    \begin{equation}
    	\mathbf{u}_{\text{DC}} = \mathbf{B}^\top\Psi^\top W_c^{-1}\left[T^\star\circ T_t^{-1} - T_t^{-1}\right],
    \end{equation}
    where:
    \begin{equation}\label{eq:transport-time}
    	T_t = \Psi W_c(t, 1)W_c^{-1}\Psi\mathbf{x} + W_c(0, t)\Psi^\top W_c^{-1}T^\star.
    \end{equation}
    In (\ref{eq:transport-time}), the two Gramians denoted as $W_c(\cdot, \cdot)$ are the "to-go" and "prior" based on the pose at $t$. The inclusion of $\mathbf{u}_{\text{DC}}$ rounds out our control law in (\ref{eq:vs-pid}), and essentially provides us with error correction to the desired depth map in the visual servo controller. Our approach can be labelled as PID control with gravity compensation in both the pose and image spaces, leading to a hybrid method that combines the PD capabilities in PBVS and the error correction in IBVS. 
    
    \begin{figure*}[t]
    	\centering
    	\includegraphics[width=.8\linewidth]{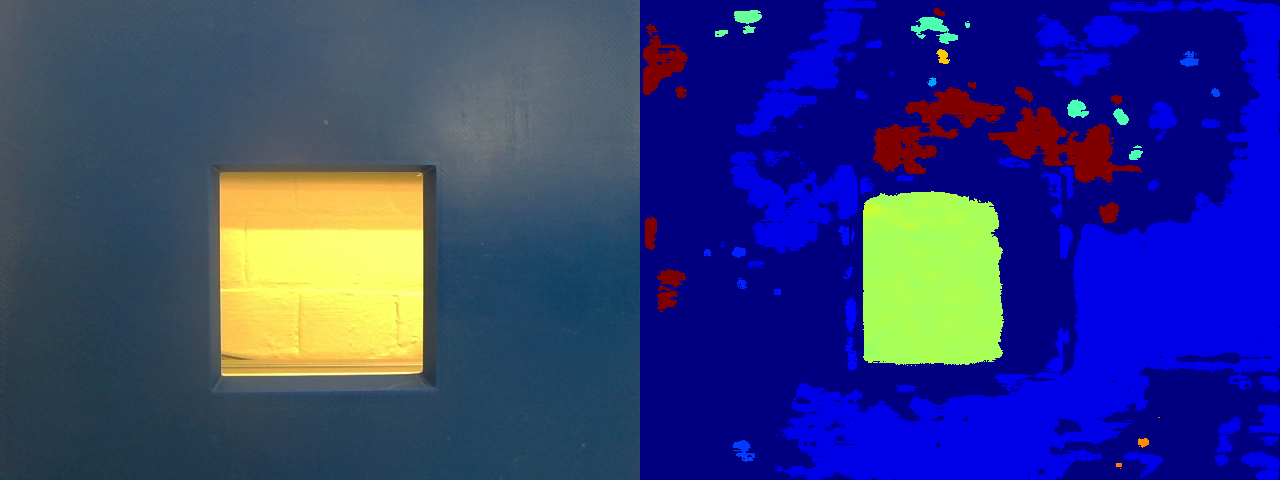}
    	\label{fig:square-hole}
    	\caption{Plots of the feature being evaluated, taken at the target pose $\mathfrak{g}_1$.}
    	\label{fig:features}
    \end{figure*}

    \section{Evaluation} \label{sec:results}

    \subsection{Overview}
    
    To validate our approach, we consider an example real-world problem that can be solved with visual servoing. We use a representative feature from the common peg-in-hole literature, shown in Figure \ref{fig:features}, which has the same target pose $\mathfrak{g}_1$ and depth map $\rho_1$ with randomised initial poses. For our robot manipulator input, the control in (\ref{eq:vs-pid}) becomes a wrench $\boldsymbol{\xi}_{\mathbf{b}, t} = \begin{bmatrix}
    \boldsymbol{v}_t, \boldsymbol{\omega}_t
    \end{bmatrix}^\top \in \mathfrak{se}^\ast(3)$:
    \begin{equation} \label{eq:wrench}
    	\mathbf{u} = \mathbf{B}^+\begin{bmatrix}
    		\boldsymbol{v}_t \\
    		\boldsymbol{\omega}_t
    	\end{bmatrix}.
    \end{equation}
    We evaluate the ability of the controller to generalise its control with four random poses. These are listed in Appendix \ref{app:exp-details}, and constitute various position and rotation distances from the desired pose. To ensure robustness to encoder and sensor noise in the robot, we assume that our system has converged to its final pose when the geodesic distance is below some threshold $\epsilon$. Our approach is summarised in the pseudocode provided in Algorithm \ref{alg:geo-vs}, and we provide the source code and experimentation details needed to replicate our setup \href{https://github.com/ManufacturingInformatics/geo-visual-servo}{here}.
    
    \begin{algorithm}[t]
    	\caption{Pseudocode for the geometric visual servo algorithm using optimal transport.}\label{alg:geo-vs}
    	\begin{algorithmic}
    		\Require $\mathbf{K}_d, \mathbf{K}_\mathbf{p}, \mathbf{K}_\mathbf{R} \succ 0$, ($\mathfrak{g}_1, \mathfrak{p}_1) \in \mathsf{T}^\ast\mathtt{SE}(3)$, $\rho_1 \in \mathcal{P}(\mathtt{SE}(3))$
    		\State $(\mathfrak{g}_0, \mathfrak{p}_0) \leftarrow$ Initial cotangent pose 
    		\State $T^\ast \leftarrow$ Compute unbalanced OT map
    		\While{$d(\mathfrak{g}_t, \mathfrak{g}_1) > \epsilon$}
    		\State $(\mathfrak{g}_t, \mathfrak{p}_t) \leftarrow$ Current state
    		\State $\mathbf{u}_{\text{ES}} \leftarrow \text{ComputeEnergyShaping}(\mathfrak{g}_t, \mathfrak{p}_t)$
    		\State $\mathbf{u}_{\text{DI}} \leftarrow \text{ComputeDamping}(\boldsymbol{\xi}_t)$
    		\State $\mathbf{u}_{\text{DC}} \leftarrow \text{ComputeDynamicOT}(\rho_t)$
    		\State $\mathbf{u} \leftarrow \mathbf{u}_{\text{ES}} +  \mathbf{u}_{\text{DI}} + \mathbf{u}_{\text{DC}}$
    		\State $d(\mathfrak{g}_t, \mathfrak{g}_1) \leftarrow$ Execute input and compute geodesic
    		\EndWhile
    	\end{algorithmic}
    \end{algorithm}
    
    \begin{figure*}[t]
    	\begin{subfigure}{.19\textwidth}
    		\centering
    		\includegraphics[width=\linewidth]{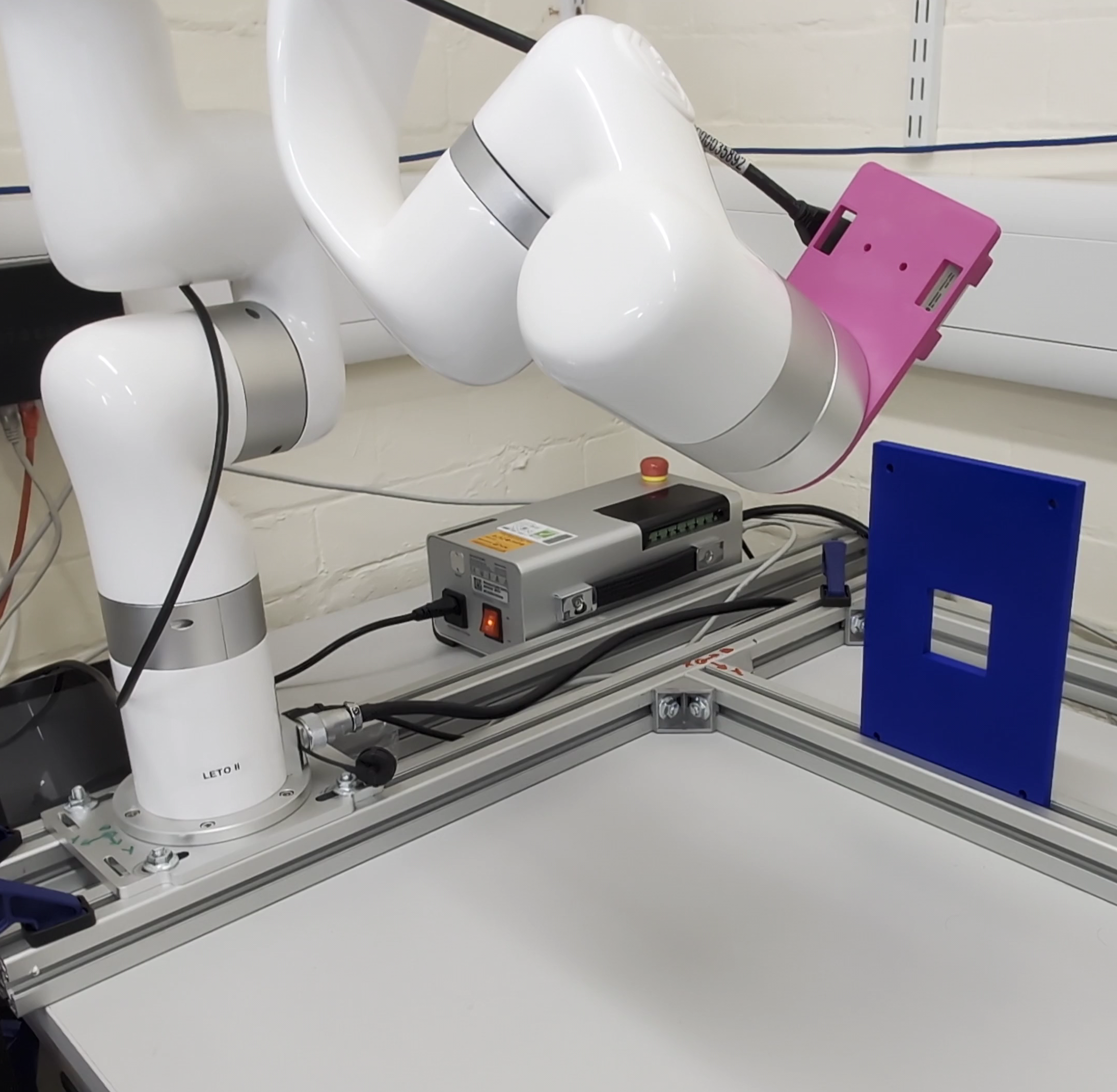}
    		\caption{$t=0$}
    		\label{fig:overlay-1}
    	\end{subfigure}%
    	\hfill
    	\begin{subfigure}{.19\textwidth}
    		\centering
    		\includegraphics[width=\linewidth]{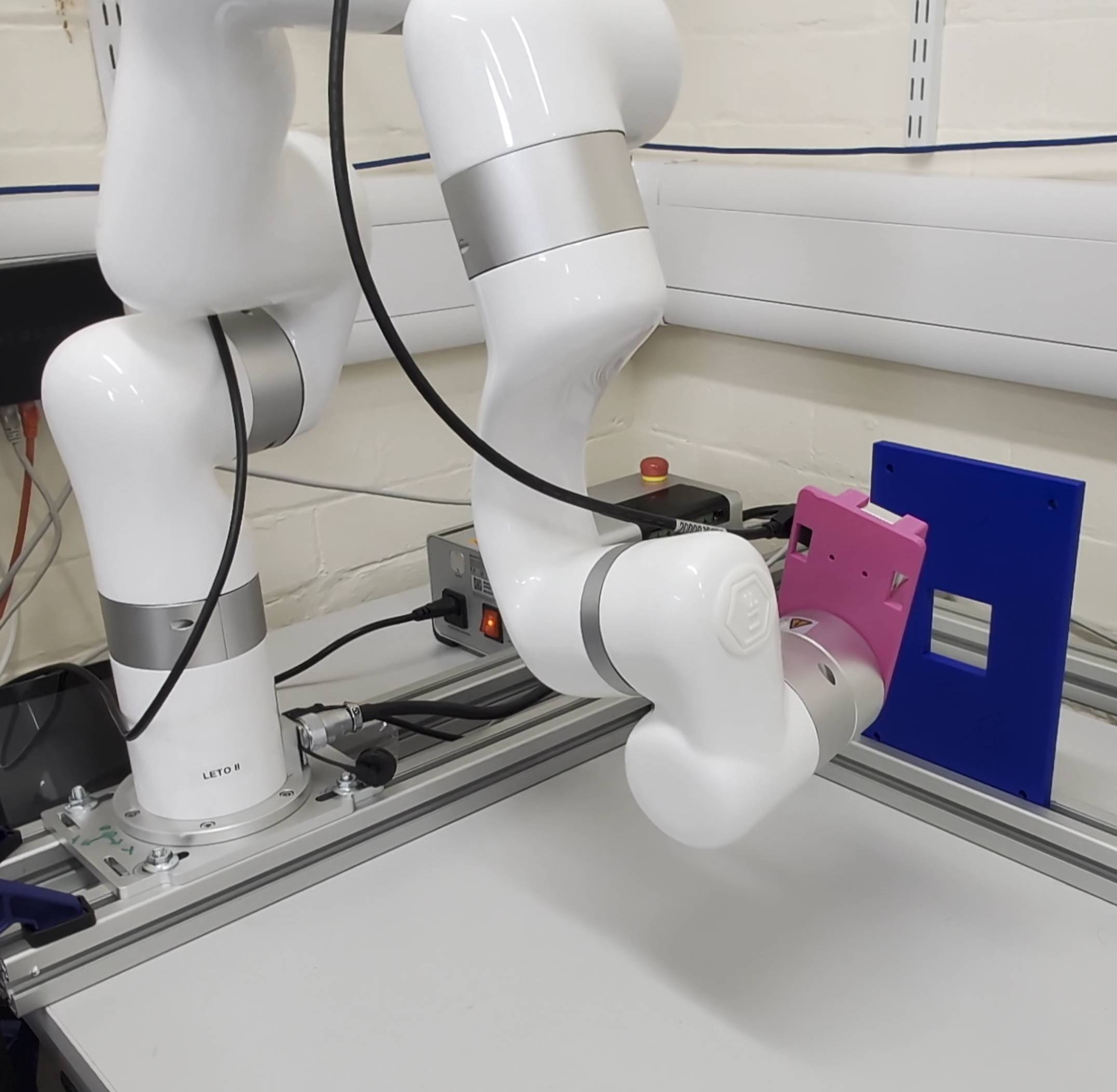}
    		\caption{$t=25$}
    		\label{fig:overlay-2}
    	\end{subfigure}%
    	\hfill
    	\begin{subfigure}{.19\textwidth}
    		\centering
    		\includegraphics[width=\linewidth]{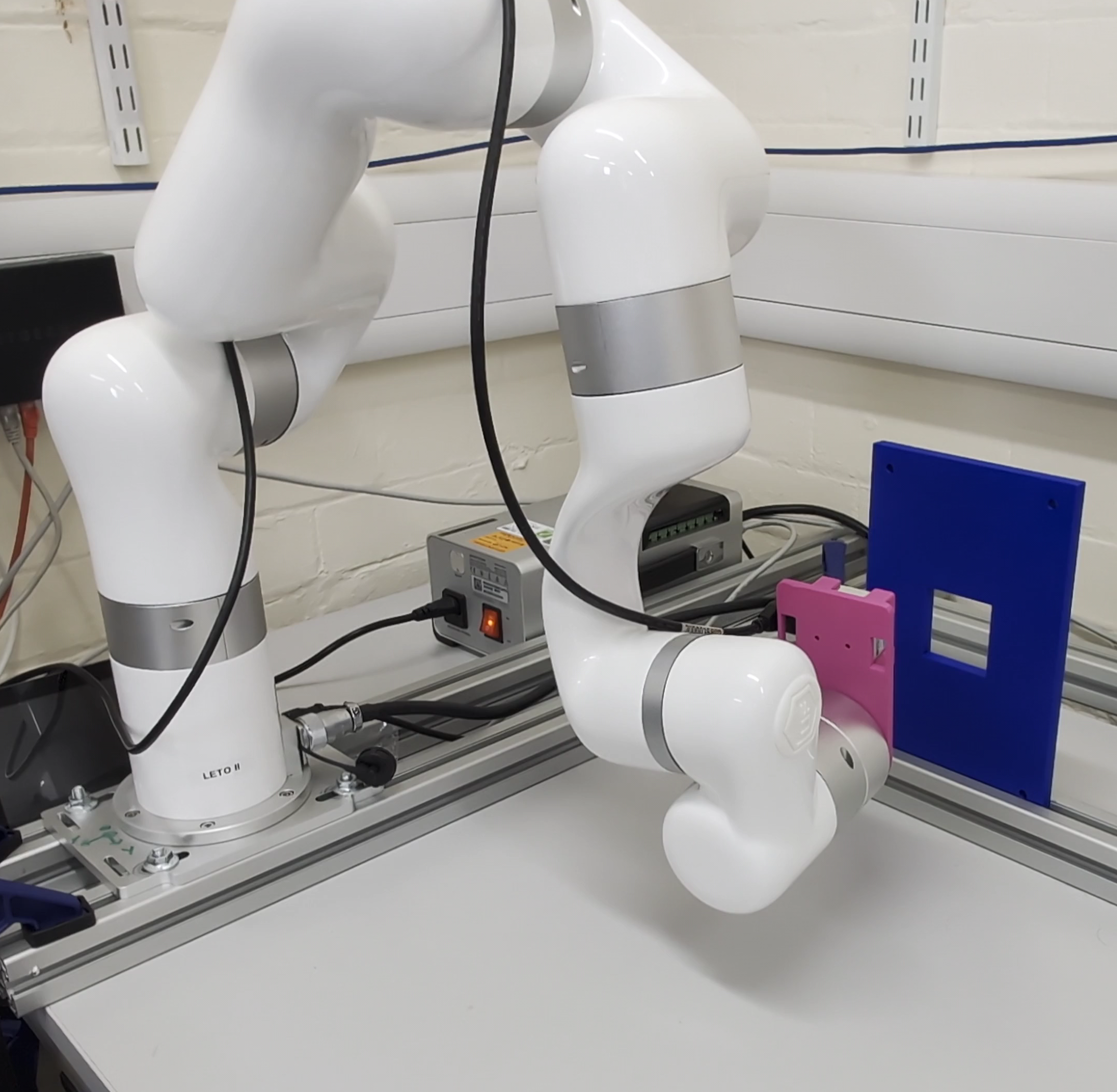}
    		\caption{$t=50$}
    		\label{fig:overlay-3}
    	\end{subfigure}%
    	\hfill
    	\begin{subfigure}{.19\textwidth}
    		\centering
    		\includegraphics[width=\linewidth]{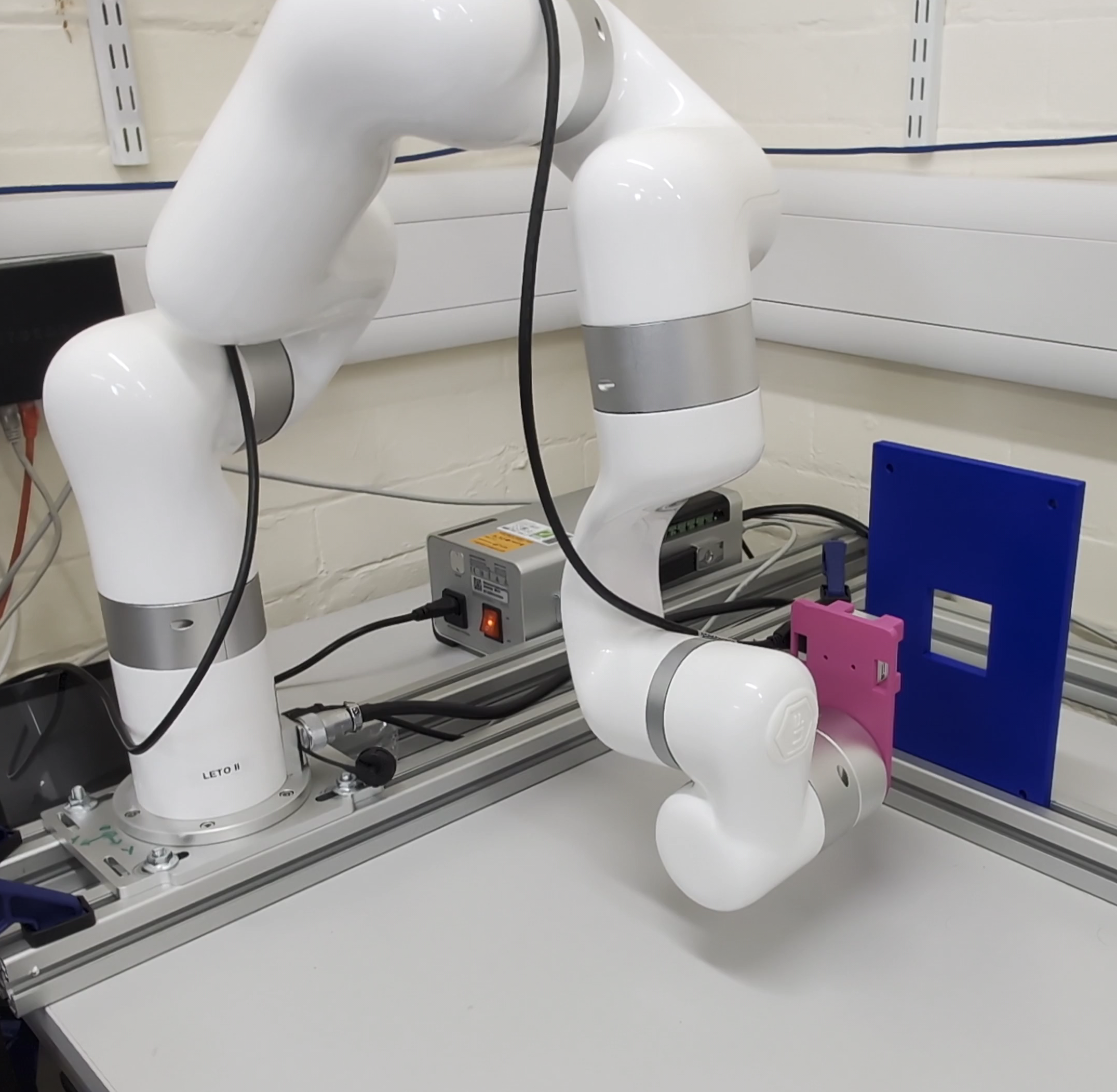}
    		\caption{$t=75$}
    		\label{fig:overlay-4}
    	\end{subfigure}%
    	\hfill
    	\begin{subfigure}{.19\textwidth}
    		\centering
    		\includegraphics[width=\linewidth]{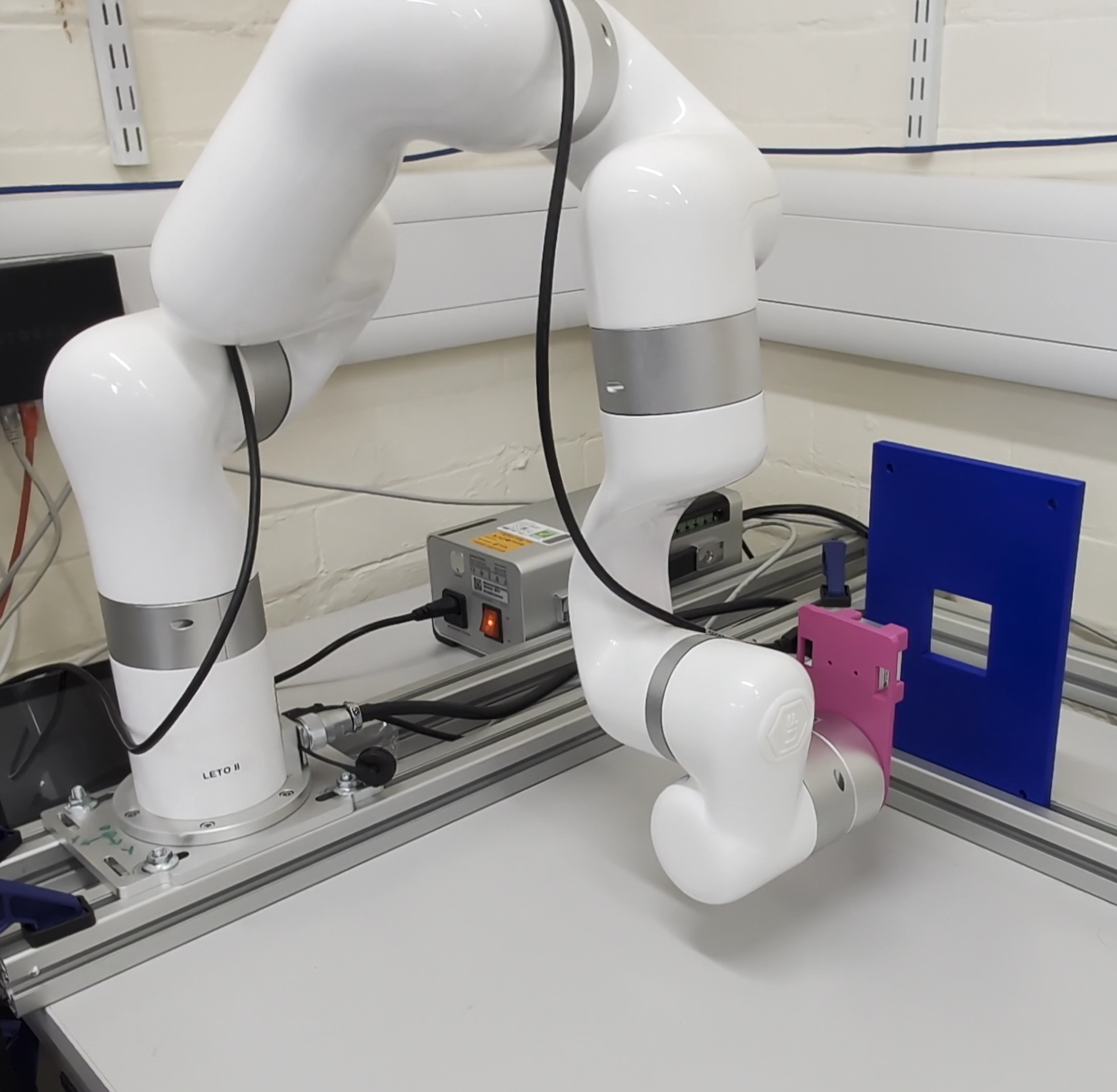}
    		\caption{$t=100$}
    		\label{fig:overlay-5}
    	\end{subfigure}%
    	\caption{Segment of the trajectory from pose 1 to the target pose $\mathfrak{g}^\ast$. As compared with Figure \ref{fig:results}, there is a fast response that slows as the robot manipulator trends towards the goal position.}
    	\label{fig:overlay}
    \end{figure*}
    
    \subsection{Results}
    
    \begin{figure*}[t]
    	\begin{subfigure}{.33\textwidth}
    		\centering
    		\includegraphics[width=.8\linewidth]{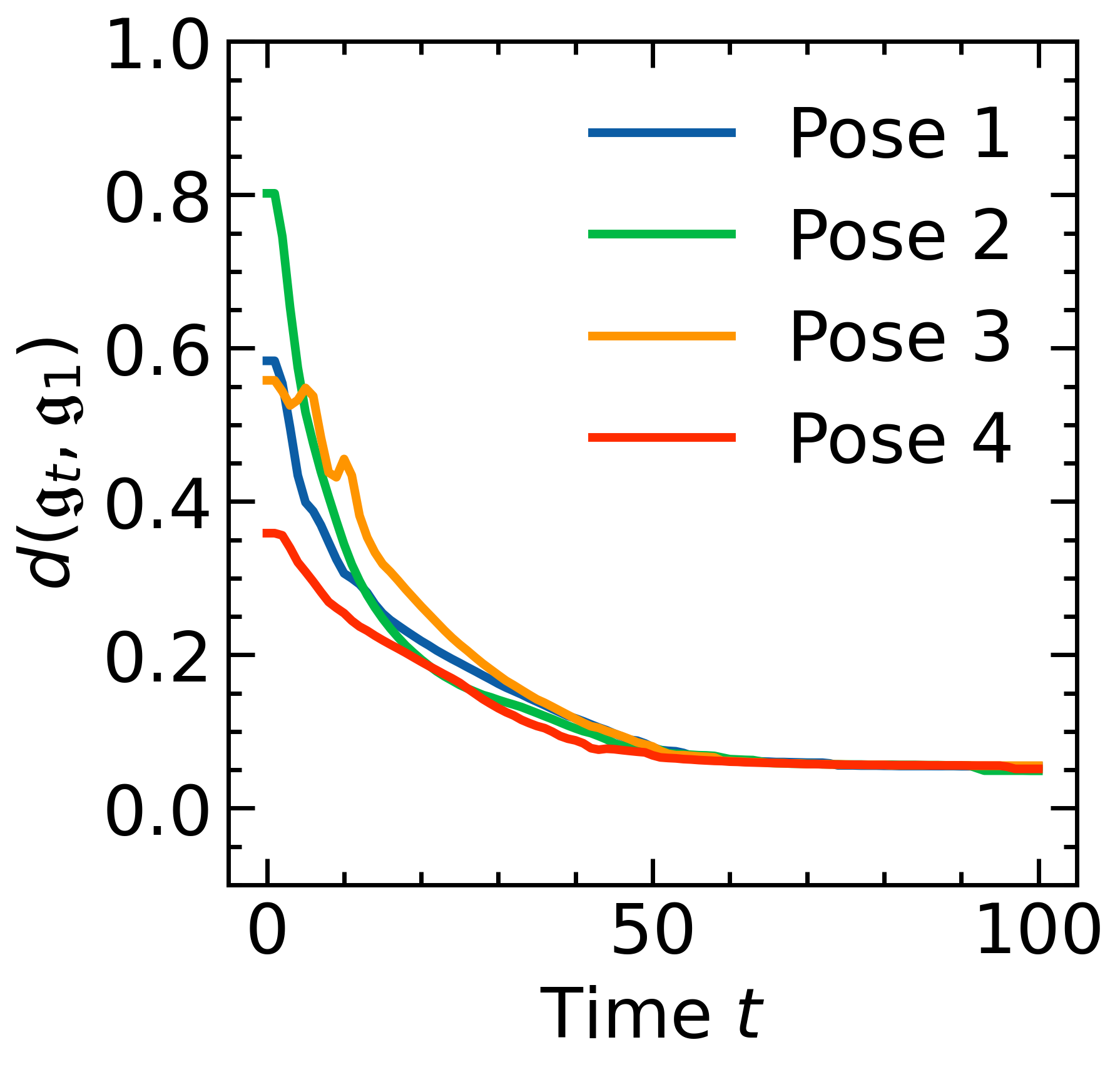}
    		\caption{}
    		\label{fig:geo-plot}
    	\end{subfigure}%
    	\begin{subfigure}{.33\textwidth}
    		\centering
    		\includegraphics[width=.84\linewidth]{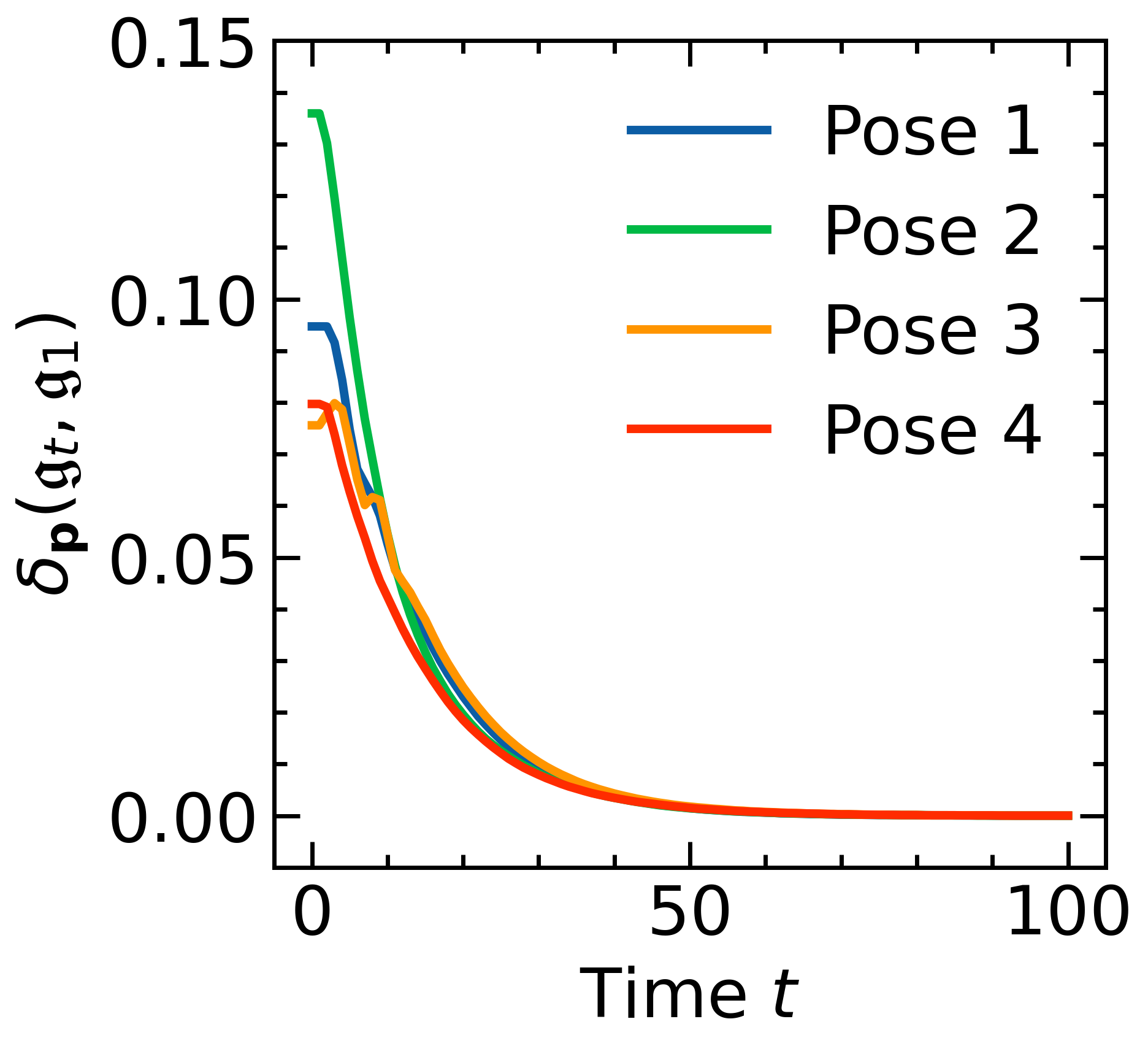}
    		\caption{}
    		\label{fig:delta-p}
    	\end{subfigure}%
    	\begin{subfigure}{.33\textwidth}
    		\centering
    		\includegraphics[width=.8\linewidth]{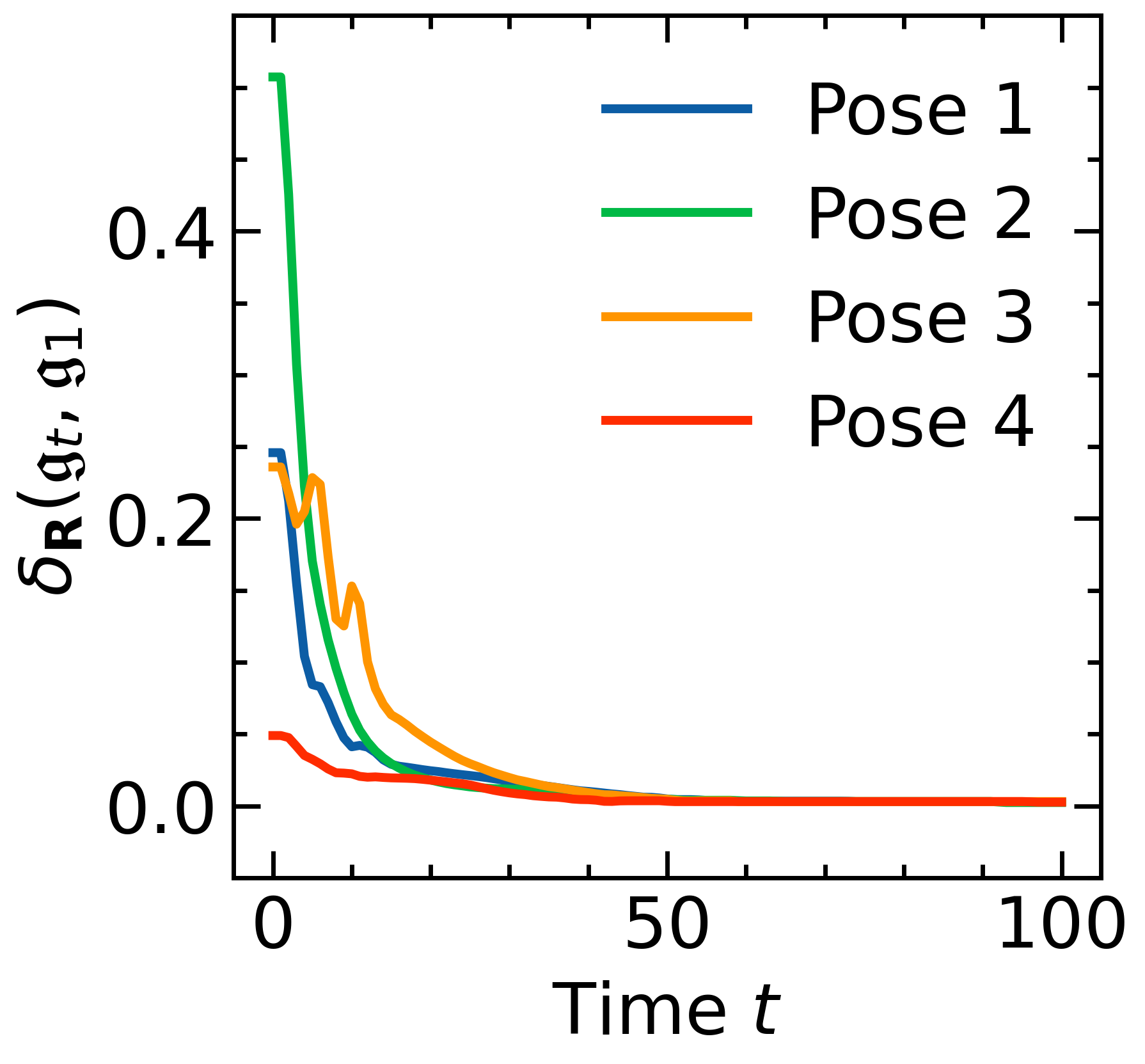}
    		\caption{}
    		\label{fig:delta-R}
    	\end{subfigure}\\
    	\centering
    	\begin{subfigure}{.33\textwidth}
    		\centering
    		\includegraphics[width=.8\linewidth]{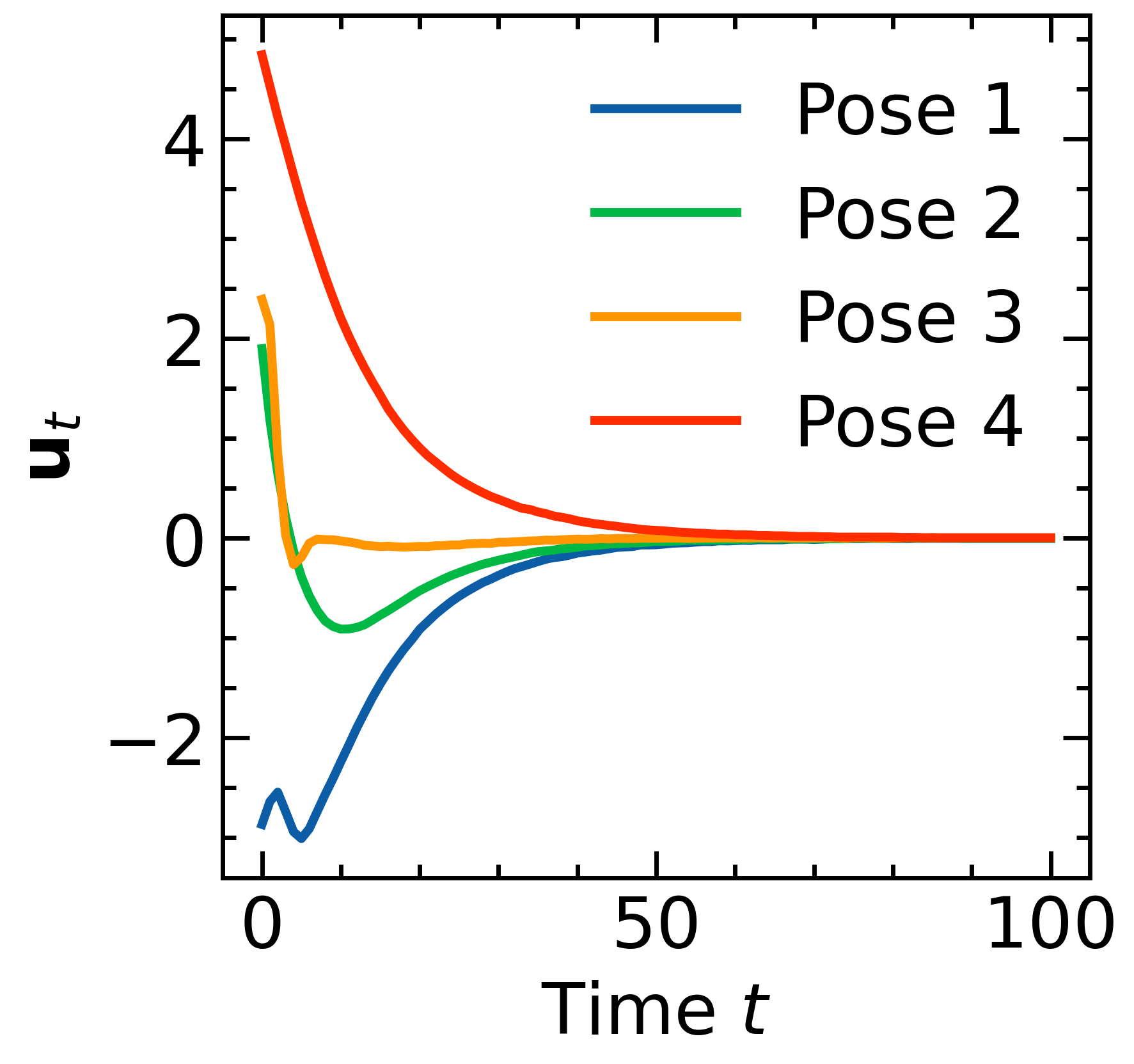}
    		\caption{}
    		\label{fig:control-input}
    	\end{subfigure}%
    	\begin{subfigure}{.33\textwidth}
    		\centering
    		\includegraphics[width=.84\linewidth]{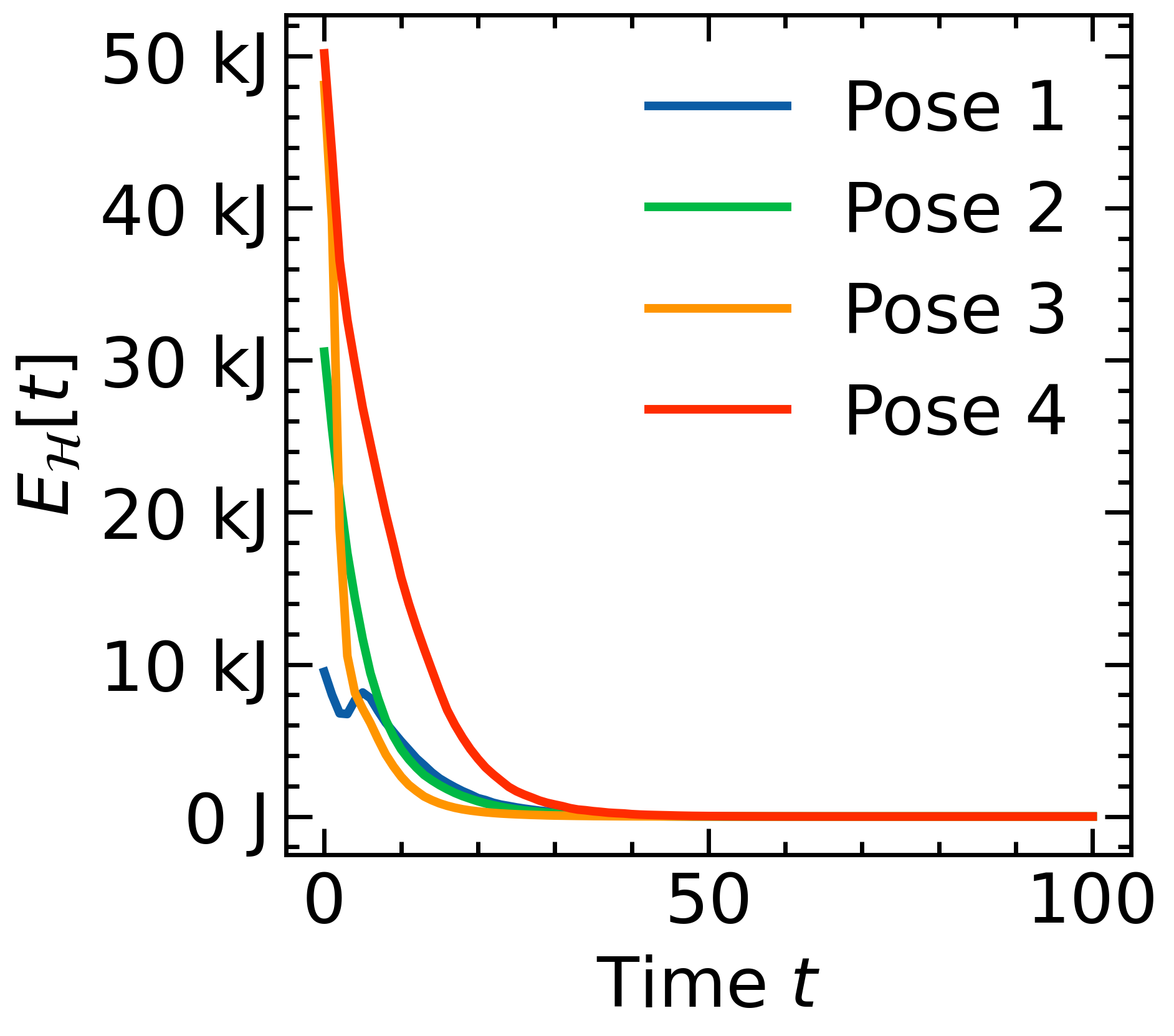}
    		\caption{}
    		\label{fig:kinetic-energy}
    	\end{subfigure}%
    	\caption{Plots of the results from 4 different initial poses to the goal pose: \textbf{(a)} The geodesic distance $d(\mathfrak{g}_t, \mathfrak{g}_1)$ in (\ref{eq:geodesic}); \textbf{(b)} The position error $\delta_\mathbf{p}(\mathfrak{g}_t, \mathfrak{g}_1)$ in (\ref{eq:delta-p}); \text{(c)} The rotational error $\delta_\mathbf{R}((\mathfrak{g}_t, \mathfrak{g}_1)$ in (\ref{eq:delta-R}); \textbf{(d)} The wrench input $\mathbf{u}_t$ from (\ref{eq:wrench}); \textbf{(e)} The kinetic energy flow $E_\mathcal{H}[t]$ from (\ref{eq:hamil-KE}).}
    	\label{fig:results}
    \end{figure*}
    
    Figure \ref{fig:overlay} showcases how the robot converges towards the target pose using the input wrench $\mathbf{u}$, with a fast initial response due to a large pose and depth map error and eventually slows as the manipulator end-effector reaches the desired pose and depth map. The path the manipulator end-effector traces through $\mathbb{R}^3$ matches the shortest path geodesic in $\mathtt{SE}(3)$ and demonstrates the controller's ability to produce task-space minimum-energy control inputs to reach the target pose. When evaluating different initial poses, Figure \ref{fig:geo-plot} showcases how the controller is capable of dealing with a variety of different initial poses and geodesic distances from the target pose, converging on zero offset from the target pose smoothly. This is also echoed in Figure \ref{fig:delta-p} and Figure \ref{fig:delta-R}, which also demonstrates how the controller dynamics can cope with rotational changes to ensure stability. Additionally, in Figure \ref{fig:delta-R}, we can see that there is a slight fluctuation in the error which can arise due to local instabilities when the Jacobian matrix becomes rank deficient and experiences rapidly increasing velocities. Despite this instability, our approach was able to recover from a potentially unstable control input. Furthermore, our controller is capable of operating with fast real-time performance as our loop time for computing control inputs is $300$ms even in the presence of rotational error disturbances and variations in the transport map control input. \\
    
    Another point of analysis is the conservation of energy when the controller is in operation. Figure \ref{fig:control-input} shows the evolution of the control input as the manipulator end-effector trends towards $\mathfrak{g}_1$, which relates to the Hamiltonian kinetic energy in (\ref{eq:hamil-KE}) through the relationship (\ref{eq:momenta}). The kinetic energy in Figure \ref{fig:kinetic-energy} showcases how the energy of the manipulator trends towards zero as $\mathfrak{p}_t \rightarrow 0$, which indicates total energy trends towards the gravity vector $\mathcal{G}(\mathfrak{g})$. From an operational perspective, this indicates that our approach would be capable of supporting extended loads that would be present in many manipulation tasks \cite{contact_rich_2024}. 
    
    \subsection{Discussion}  
    
    This work successfully reframes the visual servoing problem as a mass-movers problem, where the goal is to move a depth map from an initial distribution and pose to a desired distribution and pose. Further work has demonstrated that these control laws are capable of orientating to a variety of different target poses \cite{contact_rich_2024}. Furthermore, our method frames the control process as a flow control problem, where our flow is generated from the 2-Wasserstein distance with the energy shaping and dynamic injection providing performance guarantees regarding dynamic response. \\
    
    Further work in this area would need to address target poses that are not static (i.e. $\dot{\mathfrak{g}}_1 \neq 0$). In this instance, the boundary condition for the 2-Wasserstein distance no longer holds and the optimal transport gain cannot be computed. A way to mitigate this would be to treat the control problem as a trajectory optimisation control problem, aligning our work with the literature regarding predictive control for visual servoing \cite{vs_mpc_2019, dp_ot_2024}. \\
    
    To improve the accuracy of the depth map, LiDAR cameras would present better accuracy at greater ranges. However, the number of points and the sampling rate of the cameras would mean computing real-time transport maps is computationally expensive. Neural optimal transport methods would allow faster computation times for high-dimensional point clouds, enabling greater accuracy for systems such as aerospace manufacturing and electronics assembly \cite{Bunne_ICML_2023}. \\
    
    Another interesting area for further work is the presence of instability during the computation of the body Jacobian. As noted in proof of Theorem \ref{thm:cl-stability}, Lyapunov stability is guaranteed when the Jacobian is full rank, as the Jacobian can only have a single solution when it is singular. However, when the Jacobian is rank deficient and its rank drops below the number of joints the manipulator possesses, kinematic singularities and non-singular solutions in the computation of joint speeds occur which lead to unstable behaviour and exponentially increasing joint velocities \cite{Donelan_2007}. Avoiding areas of instability can be incorporated into the dynamics by linearising the system dynamics, or by evaluating stability equations before computing control inputs. Alternatively, one could use gain scheduling control to avoid regions of instability or through a null-space projection of the Jacobian \cite{contact_rich_2024, manipulation}. \\
    
    \section{Conclusion} \label{sec:conclusion}
    
    In this work, we presented a reframing of the visual servoing control problem in robotics by combining pose and vision target states. We leverage the geometry of port Hamiltonian systems to develop a control law for multi-body systems that incorporates geodesic shortening properties via optimal transport on point cloud data. We provide a test case using a robot manipulator on two different feature types and illustrate future areas for research. 
	
	\printbibliography[]

@book{Bunne_ICML_2023, 
    address={Hawaii, USA}, 
    type={Conference Tutorial}, 
    series={ICML 2023 Tutorials}, 
    title={{Optimal Transport in Learning, Control, and Dynamical Systems}}, 
    url={http://bunne.ch/ot\_tutorial/}, 
    institution={Hawaii Convention Center}, 
    author={Bunne, Charlotte}, 
    year={2023}, 
    month=jul, 
    pages={25}, 
    collection={ICML 2023 Tutorials}, 
    language={en},
    publisher={{Neural Information Processing Systems Foundation}}
}

@article{comp_visual_servo_2023, 
    title={A review and performance comparison of visual servoing controls}, 
    volume={7}, 
    ISSN={2366-598X}, 
    DOI={10.1007/s41315-023-00270-6}, 
    number={1}, 
    journal={International Journal of Intelligent Robotics and Applications}, 
    author={Cong, Vo Duy and Hanh, Le Duc}, 
    year={2023}, 
    month=mar, 
    pages={65–90}, 
    language={en}
}

@inbook{ot_curvature_chapter_2011, 
    address={Berlin, Heidelberg}, 
    series={Lecture Notes in Mathematics}, 
    title={Optimal Transport and Curvature}, 
    volume={2028}, 
    ISBN={978-3-642-21718-0}, 
    url={https://link.springer.com/10.1007/978-3-642-21861-3_4}, 
    DOI={10.1007/978-3-642-21861-3_4}, 
    booktitle={Nonlinear PDE’s and Applications}, 
    publisher={Springer Berlin Heidelberg}, 
    author={Figalli, Alessio and Villani, Cédric}, 
    year={2011}, 
    pages={171–217}, 
    collection={Lecture Notes in Mathematics}, 
    language={en}
}

@inbook{Donelan_2007, 
    title={Singularities of robot manipulators}, 
    ISBN={978-981-270-410-8}, 
    url={https://www.worldscientific.com/doi/abs/10.1142/9789812707499_0006}, 
    DOI={10.1142/9789812707499_0006}, 
    booktitle={Singularity Theory}, 
    publisher={WORLD SCIENTIFIC}, 
    author={Donelan, P.s.}, 
    year={2007}, 
    month=feb, 
    pages={189–217}
}

@article{duong_hamiltonian_2021, 
    title={{Port-Hamiltonian Neural ODE Networks on Lie Groups for Robot Dynamics Learning and Control}}, 
    volume={40}, 
    ISSN={1941-0468}, 
    DOI={10.1109/TRO.2024.3428433}, 
    journal={IEEE Transactions on Robotics}, 
    author={Duong, Thai and Altawaitan, Abdullah and Stanley, Jason and Atanasov, Nikolay}, 
    year={2024}, 
    pages={3695–3715}
}

@book{geo_control_2005, 
    address={New York, NY}, 
    series={Texts in Applied Mathematics}, 
    title={Geometric Control of Mechanical Systems: Modeling, Analysis, and Design for Simple Mechanical Control Systems}, 
    volume={49}, 
    rights={http://www.springer.com/tdm}, 
    ISBN={978-1-4419-1968-7}, 
    url={http://link.springer.com/10.1007/978-1-4899-7276-7}, 
    DOI={10.1007/978-1-4899-7276-7}, 
    publisher={Springer}, 
    author={Bullo, Francesco and Lewis, Andrew D.}, 
    year={2005}, 
    collection={Texts in Applied Mathematics}, 
    language={en} 
}

@book{intro_mechanics_1999, 
    address={New York, NY}, 
    series={Texts in Applied Mathematics}, 
    title={Introduction to Mechanics and Symmetry: A Basic Exposition of Classical Mechanical Systems}, 
    volume={17}, 
    rights={http://www.springer.com/tdm}, 
    ISBN={978-1-4419-3143-6}, 
    url={http://link.springer.com/10.1007/978-0-387-21792-5}, 
    DOI={10.1007/978-0-387-21792-5}, 
    publisher={Springer}, 
    author={Marsden, Jerrold E. and Ratiu, Tudor S.}, 
    year={1999}, 
    collection={Texts in Applied Mathematics}, 
    language={en}
}

@book{math_intro_manip_1994, 
    address={USA}, 
    edition={1st}, 
    title={A  Mathematical Introduction to Robotic Manipulation}, 
    ISBN={978-0-8493-7981-9}, 
    publisher={CRC Press, Inc.}, 
    author={Murray, Richard M. and Sastry, S. Shankar and Zexiang, Li}, 
    year={1994}, 
    month=feb
}

@article{note_rot_1989, 
    title={An Historical Note on Finite Rotations}, 
    volume={56}, 
    ISSN={0021-8936}, 
    DOI={10.1115/1.3176034}, 
    number={1}, 
    journal={Journal of Applied Mechanics}, 
    author={Cheng, Hui and Gupta, K. C.}, 
    year={1989}, 
    month=mar, 
    pages={139–145}
}

@book{handbook_robotics_2008, 
    address={Berlin, Heidelberg}, 
    edition={2nd}, 
    editor={Siciliano, Bruno and Khatib, Oussama},
    series={Springer Handbooks}, 
    title={Handbook of Robotics}, 
    volume={1}, 
    ISBN={978-3-540-23957-4}, 
    url={http://link.springer.com/10.1007/978-3-540-30301-5}, 
    DOI={10.1007/978-3-540-30301-5}, 
    publisher={Springer}, 
    year={2008}, 
    collection={Springer Handbooks}, 
    language={en}
}

@book{global_hamilton_2018, 
    address={Cham}, 
    series={Interaction of Mechanics and Mathematics}, 
    title={Global Formulations of Lagrangian and Hamiltonian Dynamics on Manifolds}, 
    rights={http://www.springer.com/tdm}, 
    ISBN={978-3-319-56951-2}, 
    url={http://link.springer.com/10.1007/978-3-319-56953-6}, 
    DOI={10.1007/978-3-319-56953-6}, 
    publisher={Springer International Publishing}, 
    author={Lee, Taeyoung and Leok, Melvin and McClamroch, N. Harris}, 
    year={2018}, 
    collection={Interaction of Mechanics and Mathematics}
}

@article{port_hamilton_2020, 
    title={{Port-Hamiltonian Modeling for Control}}, 
    volume={3}, 
    ISSN={2573-5144}, 
    DOI={10.1146/annurev-control-081219-092250}, 
    number={Volume 3, 2020}, 
    journal={Annual Review of Control, Robotics, and Autonomous Systems}, 
    publisher={Annual Reviews}, 
    author={Schaft, Arjan van der}, 
    year={2020}, 
    month=may, 
    pages={393–416},
    language={en}
}

@article{ot_lie_2024, 
	title={{Optimal Transport on the Lie Group of Roto-translations}}, 
	DOI={10.1137/24M1641531}, 
	journal={SIAM Journal on Imaging Sciences}, 
	publisher={Society for Industrial and Applied Mathematics}, 
	author={Bon, Daan and Pai, Gautam and Bellaard, Gijs and Mula, Olga and Duits, Remco}, 
	year={2025}, 
	month=jun, 
	pages={789–821} 
}

@article{ot_control_2021, 
    title={Optimal Transport in Systems and Control}, 
    volume={4}, 
    ISSN={2573-5144, 2573-5144}, 
    DOI={10.1146/annurev-control-070220-100858}, 
    number={1}, 
    journal={Annual Review of Control, Robotics, and Autonomous Systems}, 
    author={Chen, Yongxin and Georgiou, Tryphon T. and Pavon, Michele}, 
    year={2021}, 
    month=may, 
    pages={89–113}, 
    language={en}
}

@book{ot_book_2015, 
    address={Cham},
    series={Progress in Nonlinear Differential Equations and Their Applications}, 
    title={Optimal Transport for Applied Mathematicians: Calculus of Variations, PDEs, and Modeling}, 
    volume={87}, 
    ISBN={978-3-319-20827-5}, 
    url={https://link.springer.com/10.1007/978-3-319-20828-2}, 
    DOI={10.1007/978-3-319-20828-2}, 
    publisher={Springer International Publishing}, 
    author={Santambrogio, Filippo}, 
    year={2015}, 
    collection={Progress in Nonlinear Differential Equations and Their Applications}, 
    language={en}
}

@inproceedings{stochastic_jigs_1982, 
	title={{Applying Stochastic Control Theory to Robot Sensing, Teaching, and Long Term Control}}, 
	volume={1982-June}, 
	DOI={10.23919/ACC.1982.4788043}, 
	booktitle={Proceedings of the American Control Conference}, 
	publisher={Institute of Electrical and Electronics Engineers Inc.}, 
	author={Whitney, Daniel E. and Junkel, Eric F.}, 
	year={1982}, 
	pages={1175–1183} 
}

@article{vs_review_2023, 
    title={A review and performance comparison of visual servoing controls}, 
    volume={7}, 
    ISSN={2366-598X}, 
    DOI={10.1007/s41315-023-00270-6}, 
    number={1}, 
    journal={International Journal of Intelligent Robotics and Applications}, 
    author={Cong, Vo Duy and Hanh, Le Duc}, 
    year={2023}, 
    month=mar, 
    pages={65–90}, 
    language={en}
}

@article{vs_underwater_2022, 
    title={A review on visual servoing for underwater vehicle manipulation systems automatic control and case study}, 
    volume={260}, 
    ISSN={0029-8018}, 
    DOI={10.1016/j.oceaneng.2022.112065}, 
    journal={Ocean Engineering}, 
    author={Huang, Hai and Bian, Xinyu and Cai, Fengchun and Li, Jiyong and Jiang, Tao and Zhang, Zhenkun and Sun, Chaoyu}, 
    year={2022}, 
    month=sep, 
    pages={112065}
}

@inproceedings{vs_mpc_2019, 
    title={{Visual Servo Application Using Model Predictive Control (MPC) Method on Pan-tilt Camera Platform}}, 
    ISSN={2639-5045}, 
    url={https://ieeexplore.ieee.org/abstract/document/8916673}, 
    DOI={10.1109/ICA.2019.8916673}, 
    author={Saragih, Christ Freben Dommaris and Kinasih, Fabiola Maria Teresa Retno and Machbub, Carmadi and Rusmin, Pranoto Hidaya and Rohman, Arief Syaichu}, 
    year={2019}, 
    month=jul, 
    pages={1–7}
}

@inproceedings{vs_spacecraft_2024, 
    title={Visual Servoing for Robotic On-Orbit Servicing: A Survey},
    url={https://ieeexplore.ieee.org/document/10687516}, 
    DOI={10.1109/iSpaRo60631.2024.10687516}, 
    booktitle={2024 International Conference on Space Robotics (iSpaRo)}, 
    author={Amaya-Mejía, Lina María and Ghita, Mohamed and Dentler, Jan and Olivares-Mendez, Miguel and Martinez, Carol}, 
    year={2024}, 
    month=jun, 
    pages={178–185}
}

@article{vs_noncooperative_2024, 
    title={Two-phase visual servoing for capturing tumbling non-cooperative satellites with a space manipulator}, 
    ISSN={1000-9361}, 
    url={https://www.sciencedirect.com/science/article/pii/S1000936124001961}, 
    DOI={10.1016/j.cja.2024.05.030}, 
    journal={Chinese Journal of Aeronautics}, 
    author={Zhang, Dezhi and Yang, Guocai and Sun, Yongjun and Ji, Junhong and Jin, Minghe and Liu, Hong}, 
    year={2024}, 
    month=may 
}

@article{vs_image_moments_2004, 
    title={Image moments: a general and useful set of features for visual servoing}, 
    volume={20}, 
    ISSN={1941-0468}, 
    DOI={10.1109/TRO.2004.829463}, 
    number={4}, 
    journal={IEEE Transactions on Robotics}, 
    author={Chaumette, François}, 
    year={2004}, 
    month=aug, 
    pages={713–723} 
}

@article{safe_learning_2024, 
    title={Safe Learning in Robotics: From Learning-Based Control to Safe Reinforcement Learning}, 
    volume={5}, 
    ISSN={2573-5144}, 
    DOI={10.1146/annurev-control-042920-020211}, 
    number={Volume 5, 2022}, 
    journal={Annual Review of Control, Robotics, and Autonomous Systems}, 
    publisher={Annual Reviews}, 
    author={Brunke, Lukas and Greeff, Melissa and Hall, Adam W. and Yuan, Zhaocong and Zhou, Siqi and Panerati, Jacopo and Schoellig, Angela P.}, 
    year={2022}, 
    month=may, 
    pages={411–444}, 
    language={en}
}

@article{sensor_review_2021, 
    title={Sensor-Based Control for Collaborative Robots: Fundamentals, Challenges, and Opportunities}, 
    volume={14}, 
    ISSN={1662-5218}, 
    url={https://www.frontiersin.org/journals/neurorobotics/articles/10.3389/fnbot.2020.576846/full}, 
    DOI={10.3389/fnbot.2020.576846}, 
    journal={Frontiers in Neurorobotics}, 
    publisher={Frontiers}, 
    author={Cherubini, Andrea and Navarro-Alarcon, David}, 
    year={2021}, 
    month=jan, 
    language={English}
}

@article{vision_review_2015, 
    title={A review on vision-based control of flexible manipulators}, 
    volume={29}, 
    ISSN={0169-1864, 1568-5535}, 
    DOI={10.1080/01691864.2015.1078743}, 
    number={24}, 
    journal={Advanced Robotics}, 
    author={Hussein, M.T.}, 
    year={2015}, 
    month=dec, 
    pages={1575–1585}, 
    language={en}
}

@article{2d_vs_1999, 
    title={{2 1/2 D visual servoing}}, 
    volume={15}, 
    ISSN={2374-958X}, 
    DOI={10.1109/70.760345}, 
    number={2}, 
    journal={IEEE Transactions on Robotics and Automation}, 
    author={Malis, E. and Chaumette, F. and Boudet, S.}, 
    year={1999}, 
    month=apr, 
    pages={238–250} 
}

@article{history_manip_2022, 
    title={An Historical Perspective on the Control of Robotic Manipulators}, 
    volume={5}, 
    ISSN={2573-5144}, 
    DOI={10.1146/annurev-control-042920-094829},
    number={Volume 5, 2022}, 
    journal={Annual Review of Control, Robotics, and Autonomous Systems}, 
    publisher={Annual Reviews}, 
    author={Spong, Mark W.}, 
    year={2022}, 
    month=may, 
    pages={1–31}, 
    language={en} 
}

@inproceedings{deep_vs_2022, 
    title={{DFBVS: Deep Feature-Based Visual Servo}}, 
    ISSN={2161-8089}, 
    DOI={10.1109/CASE49997.2022.9926560}, 
    booktitle={2022 IEEE 18th International Conference on Automation Science and Engineering (CASE)}, 
    author={Adrian, Nicholas and Do, Van-Thach and Pham, Quang-Cuong}, 
    year={2022}, 
    month=aug, 
    pages={1783–1789}
}

@inproceedings{ot_change_detect_2023, 
    title={{Optimal Transport for Change Detection on LIDAR Point Clouds}}, 
    ISSN={2153-7003}, 
    url={https://ieeexplore.ieee.org/document/10283101}, 
    DOI={10.1109/IGARSS52108.2023.10283101},
    booktitle={IGARSS 2023 - 2023 IEEE International Geoscience and Remote Sensing Symposium}, 
    author={Fiorucci, Marco and Naylor, Peter and Yamada, Makoto}, 
    year={2023}, 
    month=jul, 
    pages={982–985} 
}

@inproceedings{ot_metric_point_2024, 
    title={{Metric Learning for 3D Point Clouds Using Optimal Transport}}, 
    ISSN={2690-621X}, 
    url={https://ieeexplore.ieee.org/document/10495643}, 
    DOI={10.1109/WACVW60836.2024.00063}, 
    booktitle={2024 IEEE/CVF Winter Conference on Applications of Computer Vision Workshops (WACVW)}, 
    author={Katageri, Siddharth and Sarkar, Srinjay and Sharma, Charu}, 
    year={2024}, 
    month=jan, 
    pages={552–560} 
}

@article{ot_motion_2023, 
    title={Accelerating Motion Planning via Optimal Transport}, 
    volume={36}, 
    journal={Advances in Neural Information Processing Systems}, 
    author={Le, An T. and Chalvatzaki, Georgia and Biess, Armin and Peters, Jan R.}, 
    year={2023}, 
    month=dec, 
    pages={78453–78482}, 
    language={en} 
}

@article{ot_bb_2000, 
    title={{A computational fluid mechanics solution to the Monge-Kantorovich mass transfer problem}}, 
    volume={84}, 
    ISSN={0945-3245}, 
    DOI={10.1007/s002110050002}, 
    number={3}, 
    journal={Numerische Mathematik}, 
    author={Benamou, Jean-David and Brenier, Yann}, 
    year={2000}, 
    month=jan, 
    pages={375–393}, 
    language={en} 
}

@article{ot_averaged_2023, 
    title={{Optimal Transport for Averaged Control}}, 
    volume={7}, 
    ISSN={2475-1456}, 
    DOI={10.1109/LCSYS.2022.3222744}, 
    journal={IEEE Control Systems Letters}, 
    author={Adu, Daniel Owusu}, 
    year={2023}, 
    pages={727–732} 
}

@article{ot_linear_2017, 
    title={{Optimal Transport Over a Linear Dynamical System}}, 
    volume={62}, 
    ISSN={1558-2523}, 
    DOI={10.1109/TAC.2016.2602103}, 
    number={5}, 
    journal={IEEE Transactions on Automatic Control}, 
    author={Chen, Yongxin and Georgiou, Tryphon T. and Pavon, Michele}, 
    year={2017}, 
    month=may, 
    pages={2137–2152}
}

@article{ot_applications_2021, 
    title={{Optimal Transport for Applications in Control and Estimation}}, 
    volume={41}, 
    ISSN={1941-000X}, 
    DOI={10.1109/MCS.2021.3076390}, 
    number={4}, 
    journal={IEEE Control Systems Magazine}, 
    author={Chen, Yongxin and Karlsson, Johan and Ringh, Axel}, 
    year={2021}, 
    month=aug, 
    pages={28–33} 
}

@article{ot_steering_2016, 
    title={{Optimal Steering of a Linear Stochastic System to a Final Probability Distribution, Part I}}, 
    volume={61}, 
    ISSN={1558-2523}, 
    DOI={10.1109/TAC.2015.2457784}, 
    number={5}, 
    journal={IEEE Transactions on Automatic Control}, 
    author={Chen, Yongxin and Georgiou, Tryphon T. and Pavon, Michele}, 
    year={2016}, 
    month=may, 
    pages={1158–1169}
}

@article{ot_manifolds_2024, 
    title={{Linearized Optimal Transport on Manifolds}}, 
    ISSN={0036-1410}, 
    DOI={10.1137/23M1564535}, 
    journal={SIAM Journal on Mathematical Analysis}, 
    publisher={Society for Industrial and Applied Mathematics},
    author={Sarrazin, Clément and Schmitzer, Bernhard}, 
    year={2024}, 
    month=aug, 
    pages={4970–5016} 
}

@article{screw_lie_2018, 
    title={{Screw and Lie group theory in multibody kinematics}}, 
    volume={43}, 
    ISSN={1573-272X}, 
    DOI={10.1007/s11044-017-9582-7}, 
    number={1}, 
    journal={Multibody System Dynamics}, 
    author={Müller, Andreas}, 
    year={2018}, 
    month=may, 
    pages={37–70}, 
    language={en} 
}

@article{contact_rich_2024, 
    title={{Contact-Rich SE(3)-Equivariant Robot Manipulation Task Learning via Geometric Impedance Control}}, 
    volume={9}, 
    ISSN={2377-3766}, 
    DOI={10.1109/LRA.2023.3346748}, 
    number={2}, 
    journal={IEEE Robotics and Automation Letters}, 
    author={Seo, Joohwan and Prakash, Nikhil P. S. and Zhang, Xiang and Wang, Changhao and Choi, Jongeun and Tomizuka, Masayoshi and Horowitz, Roberto}, 
    year={2024},
    month=feb, 
    pages={1508–1515} 
}

@article{hamilton_pendulum_2024, 
    title={{Assessing an Energy-Based Control for the Soft Inverted Pendulum in Hamiltonian Form}}, 
    volume={8}, 
    ISSN={2475-1456}, 
    DOI={10.1109/LCSYS.2024.3405410}, 
    journal={IEEE Control Systems Letters}, 
    author={Pagnanelli, Giulia and Pierallini, Michele and Angelini, Franco and Bicchi, Antonio}, 
    year={2024}, 
    pages={922–927} 
}

@inproceedings{robot_hamilton_2016, 
    address={Svratka, Czech Republic}, 
    title={{Robot Control in terms of Hamiltonian Mechanics}}, 
    booktitle={22nd International Conference on Engineering Mechanics}, 
    author={Záda, V and Belda, K}, 
    year={2016}, 
    month=may, 
    pages={627–630}, 
    language={en} 
}

@article{port_hamilton_overview_2014, 
    title={{Port-Hamiltonian Systems Theory: An Introductory Overview}}, 
    volume={1}, 
    ISSN={2325-6818, 2325-6826}, 
    DOI={10.1561/2600000002}, 
    number={2–3}, 
    journal={Foundations and Trends® in Systems and Control}, 
    author={Van Der Schaft, Arjan and Jeltsema, Dimitri}, 
    year={2014}, 
    pages={173–378}, 
    language={en} 
}

@inbook{schaft_port_2004, 
    address={Vienna}, 
    title={{Port-Hamiltonian Systems: Network Modeling and Control of Nonlinear Physical Systems}},
    ISBN={978-3-211-22867-8}, 
    url={http://link.springer.com/10.1007/978-3-7091-2774-2_9}, 
    DOI={10.1007/978-3-7091-2774-2_9}, 
    booktitle={Advanced Dynamics and Control of Structures and Machines}, 
    publisher={Springer Vienna}, 
    author={Schaft, A. J.}, 
    editor={Irschik, Hans and Schlacher, Kurt}, 
    year={2004}, 
    pages={127–167}, 
    language={en}
}

@article{duong_csl_2022, 
    title={{Adaptive Control of SE(3) Hamiltonian Dynamics With Learned Disturbance Features}}, 
    volume={6}, 
    ISSN={2475-1456}, 
    DOI={10.1109/LCSYS.2022.3177156}, 
    journal={IEEE Control Systems Letters}, 
    author={Duong, Thai and Atanasov, Nikolay}, 
    year={2022}, 
    pages={2773–2778} 
}

@article{ph_stabil_2002, 
    title={Stabilization of a class of underactuated mechanical systems via interconnection and damping assignment}, 
    volume={47}, 
    ISSN={0018-9286}, 
    url={http://dx.doi.org/10.1109/TAC.2002.800770}, 
    DOI={10.1109/tac.2002.800770}, 
    number={8}, 
    journal={IEEE Transactions on Automatic Control}, 
    publisher={Institute of Electrical and Electronics Engineers (IEEE)}, 
    author={Ortega, R. and Spong, M.W. and Gomez-Estern, F. and Blankenstein, G.}, 
    year={2002}, 
    month=aug, 
    pages={1218–1233}
}

@book{mechanics_1960,
  author={Landau, Lev D. and Lifshitz, Evgency M.},
  editor={Sykes, J. B. and Bell, J. S.},
  publisher={Pergamon Press},
  title={{Course on Theoretical Physics, Volume I: Mechanics}},
  year={1960}
}

@article{image_moments_2004, 
    title={Image moments: a general and useful set of features for visual servoing}, 
    volume={20}, 
    ISSN={1941-0468}, 
    DOI={10.1109/TRO.2004.829463}, 
    number={4}, 
    journal={IEEE Transactions on Robotics}, 
    author={Chaumette, Francois}, 
    year={2004}, 
    month=aug, 
    pages={713–723} 
}

@article{feature_tracking_2005, 
    title={Feature tracking for visual servoing purposes}, 
    volume={52}, 
    ISSN={09218890}, 
    DOI={10.1016/j.robot.2005.03.009}, 
    number={1}, 
    journal={Robotics and Autonomous Systems}, 
    author={Marchand, Éric and Chaumette, François}, 
    year={2005}, 
    month=jul, 
    pages={53–70}, 
    language={en} 
}

@article{vs_pipe_2019, 
    title={Structured Light-Based Visual Servoing for Robotic Pipe Welding Pose Optimization}, 
    volume={7}, 
    ISSN={2169-3536}, 
    DOI={10.1109/ACCESS.2019.2943248}, 
    journal={IEEE Access}, 
    author={Li, Jinquan and Chen, Zhe and Rao, Gang and Xu, Jing}, 
    year={2019}, 
    pages={138327–138340} 
}

@inproceedings{vs_multi_camera_2011, 
    title={Multi-sensor data fusion in sensor-based control: Application to multi-camera visual servoing}, 
    ISSN={1050-4729}, 
    DOI={10.1109/ICRA.2011.5979715}, 
    booktitle={2011 IEEE International Conference on Robotics and Automation}, 
    publisher={IEEE}, 
    author={Kermorgant, Olivier and Chaumette, François}, 
    year={2011}, 
    month=may, 
    pages={4518–4523} 
}

@article{vs_medical_2014, 
    title={Visual servoing in medical robotics: a survey. Part I: endoscopic and direct vision imaging – techniques and applications}, 
    volume={10}, 
    ISSN={1478-596X}, 
    DOI={10.1002/rcs.1531}, 
    number={3}, 
    journal={The International Journal of Medical Robotics and Computer Assisted Surgery}, 
    author={Azizian, Mahdi and Khoshnam, Mahta and Najmaei, Nima and Patel, Rajni V.}, 
    year={2014}, 
    pages={263–274}, 
    language={en} 
}

@article{vs_image_based_2022, 
    title={Image-based visual servoing with depth estimation}, 
    volume={44}, 
    ISSN={0142-3312}, 
    DOI={10.1177/01423312211064681}, 
    number={9}, 
    journal={Transactions of the Institute of Measurement and Control}, 
    publisher={SAGE Publications Ltd STM}, 
    author={Gongye, Qingxuan and Cheng, Peng and Dong, Jiuxiang}, 
    year={2022}, 
    month=jun, 
    pages={1811–1823}, 
    language={en} 
}

@article{vs_depth_map_2014, 
    title={A Dense and Direct Approach to Visual Servoing Using Depth Maps}, 
    volume={30}, 
    ISSN={1941-0468}, 
    DOI={10.1109/TRO.2014.2325991}, 
    number={5},
    journal={IEEE Transactions on Robotics}, 
    author={Teulière, Céline and Marchand, Eric}, 
    year={2014}, 
    month=oct, 
    pages={1242–1249} 
}

@inproceedings{vs_point_cloud_2020, 
    title={A Visual Servoing Method based on Point Cloud}, 
    DOI={10.1109/RCAR49640.2020.9303277}, 
    booktitle={2020 IEEE International Conference on Real-time Computing and Robotics (RCAR)}, 
    author={Zhang, Shiyi and Gong, Zeyu and Tao, Bo and Ding, Han}, 
    year={2020}, 
    month=sep, 
    pages={369–374} 
}

@article{robust_vs_2016, 
    title={Robust Online Model Predictive Control for a Constrained Image-Based Visual Servoing}, 
    volume={63}, 
    ISSN={02780046}, 
    DOI={10.1109/TIE.2015.2510505}, 
    number={4}, 
    journal={IEEE Transactions on Industrial Electronics}, 
    publisher={Institute of Electrical and Electronics Engineers Inc.}, 
    author={Hajiloo, Amir and Keshmiri, Mohammad and Xie, Wen Fang and Wang, Ting Ting}, 
    year={2016}, 
    month=apr, 
    pages={2242–2250} 
}

@inproceedings{mpc_quad_2020, 
    title={Fast Model Predictive Image-Based Visual Servoing for Quadrotors}, 
    ISSN={2153-0866}, 
    url={https://ieeexplore.ieee.org/document/9340759}, 
    DOI={10.1109/IROS45743.2020.9340759}, 
    booktitle={2020 IEEE/RSJ International Conference on Intelligent Robots and Systems (IROS)}, 
    author={Roque, Pedro and Bin, Elisa and Miraldo, Pedro and Dimarogonas, Dimos V.}, 
    year={2020}, 
    month=oct, 
    pages={7566–7572} 
}

@article{diffopt_2024, 
    title={DiffOcclusion: Differentiable Optimization Based Control Barrier Functions for Occlusion-Free Visual Servoing}, 
    volume={9}, 
    ISSN={2377-3766}, 
    DOI={10.1109/LRA.2024.3364468}, 
    number={4}, 
    journal={IEEE Robotics and Automation Letters}, 
    author={Wei, Shiqing and Dai, Bolun and Khorrambakht, Rooholla and Krishnamurthy, Prashanth and Khorrami, Farshad}, 
    year={2024}, 
    month=apr, 
    pages={3235–3242} 
}

@inproceedings{dome_2022, 
    address={Kyoto, Japan}, 
    title={Demonstrate Once, Imitate Immediately (DOME): Learning Visual Servoing for One-Shot Imitation Learning}, 
    url={https://www.robot-learning.uk/dome}, 
    DOI={10.48550/arXiv.2204.02863}, 
    booktitle={Proc. 2022 International Conference on Intelligent Robots and Systems}, 
    publisher={IEEE}, 
    author={Valassakis, Eugene and Papagiannis, Georgios and Di Palo, Norman and Johns, Edward}, 
    year={2022}, 
    pages={1–8} 
}

@article{vs_ssl_2022, 
    title={Self-Supervised Learning of Visual Servoing for Low-Rigidity Robots Considering Temporal Body Changes}, 
    volume={7}, 
    ISSN={2377-3766}, 
    DOI={10.1109/LRA.2022.3186074}, 
    number={3}, 
    journal={IEEE Robotics and Automation Letters}, 
    author={Kawaharazuka, Kento and Kanazawa, Naoaki and Okada, Kei and Inaba, Masayuki}, 
    year={2022}, 
    month=jul, 
    pages={7881–7887} 
}

@inproceedings{vs_aero_2020, 
    title={Comparing Position- and Image-Based Visual Servoing for Robotic Assembly of Large Structures}, 
    ISSN={2161-8089}, 
    DOI={10.1109/CASE48305.2020.9217028}, 
    booktitle={2020 IEEE 16th International Conference on Automation Science and Engineering (CASE)}, 
    author={Peng, Yuan-Chih and Jivani, Devavrat and Radke, Richard J. and Wen, John}, 
    year={2020}, 
    month=aug, 
    pages={1608–1613} 
}

@article{comp_ot_2019, 
    title={{Computational Optimal Transport: With Applications to Data Science}}, 
    volume={11}, 
    ISSN={1935-8237, 1935-8245}, 
    DOI={10.1561/2200000073}, 
    number={5–6}, 
    journal={Foundations and Trends® in Machine Learning}, 
    publisher={Now Publishers, Inc.}, 
    author={Peyré, Gabriel and Cuturi, Marco}, 
    year={2019}, 
    month=feb, 
    pages={355–607}, 
    language={English} 
}

@article{vehicle_dynamics_1999,
    author = {S Hegazy and H Rahnejat and K Hussain},
    title ={Multi-body dynamics in full-vehicle handling analysis},
    journal = {Proceedings of the Institution of Mechanical Engineers, Part K: Journal of Multi-body Dynamics},
    volume = {213},
    number = {1},
    pages = {19-31},
    year = {1999},
    doi = {10.1243/1464419991544027}
}

@article{continuum_dynamics_2024, 
    title={{Assessing an Energy-Based Control for the Soft Inverted Pendulum in Hamiltonian Form}},
    volume={8}, 
    ISSN={2475-1456}, 
    DOI={10.1109/LCSYS.2024.3405410}, 
    journal={IEEE Control Systems Letters}, 
    author={Pagnanelli, Giulia and Pierallini, Michele and Angelini, Franco and Bicchi, Antonio}, 
    year={2024}, 
    pages={922–927} 
}

@book{nonlinear_sys_1999,
    author={Sastry, Shankar},
    editor={Marsden, J.E. and Sirovich, L. and Wiggins, S.},
    publisher={Springer},
    title={{Nonlinear Systems: Analysis Stability and Control}},
    year={1999}
}

@inproceedings{canzini_iros_2024, 
    title={{Generating Continuous Paths On Learned Constraint Manifolds Using Policy Search}}, 
    ISSN={2153-0866},
    DOI={10.1109/IROS58592.2024.10802531}, 
    booktitle={2024 IEEE/RSJ International Conference on Intelligent Robots and Systems (IROS)}, 
    author={Canzini, Ethan and Pope, Simon and Tiwari, Ashutosh}, 
    year={2024}, 
    month=oct, 
    pages={5396–5401} 
}

@inproceedings{pontryagin_riemannian_2015, 
    address={Cham}, 
    title={{Pontryagin Calculus in Riemannian Geometry}}, 
    ISBN={978-3-319-25040-3}, 
    DOI={10.1007/978-3-319-25040-3_58}, 
    booktitle={Geometric Science of Information}, 
    publisher={Springer International Publishing}, 
    author={Dubois, François and Fortuné, Danielle and Rojas Quintero, Juan Antonio and Vallée, Claude}, 
    editor={Nielsen, Frank and Barbaresco, Frédéric}, 
    year={2015}, 
    pages={541–549}, 
    language={en} 
}

@article{lie_metrics_2024, 
	title={Metrics proposed for measuring the distance between two rigid-body poses: review, comparison, and combination}, 
	volume={42}, 
	ISSN={0263-5747, 1469-8668}, 
	DOI={10.1017/S0263574723001388}, 
	number={1}, 
	journal={Robotica}, 
	author={Gregorio, Raffaele Di}, 
	year={2024}, 
	month=jan, 
	pages={302–318}, 
	language={en} 
}

@article{dp_ot_2024,
	title={Dynamic Programming in Probability Spaces via Optimal Transport}, 
	volume={62}, 
	ISSN={0363-0129}, 
	DOI={10.1137/23M1560902}, 
	number={2}, 
	journal={SIAM Journal on Control and Optimization}, 
	publisher={Society for Industrial and Applied Mathematics}, 
	author={Terpin, Antonio and Lanzetti, Nicolas and Dörfler, Florian}, 
	year={2024}, 
	month=apr, 
	pages={1183–1206} 
}

@book{manipulation,
	title={{Robotic Manipulation}},
	subtitle={{Perception, Planning, and Control}},
	howpublished={{Perception, Planning, and Control"}},
	author={{Tedrake, Russ}},
	year=2024,
	url={{http://manipulation.mit.edu}}
}
	
	\appendix
	
	\section{Conservation of Energy}\label{app:energy-conservation}

    \begin{proof}\label{proof:kinetic}
        To model the Hamiltonian for energy conservation, it must hold that the Hamiltonian equals the sum of energies in the system such that \cite{mechanics_1960}:
        \begin{equation}
            \mathcal{H} = \mathcal{K} + \mathcal{G}.
        \end{equation}
        From (\ref{eq:kinetic}) and (\ref{eq:lie-state}), we see that the change of variables $(\mathfrak{g}, \dot{\mathfrak{g}}) \mapsto (\mathfrak{g}, \mathfrak{p})$ from the Legendre transform (\ref{eq:legendre-hamilton}) alters the kinetic energy $\mathcal{K}$ only. By observation, we see that the kinetic energy term in (\ref{eq:hamil-manip}) is the phase-space denotation of kinetic energy.
    
        Using Definition \ref{def:lie-poisson-bracket}, by the multivariate chain rule:
        \begin{equation}
            \begin{aligned}
                \frac{\text{d}\mathcal{H}}{\text{d}t}& = \frac{\partial\mathcal{H}}{\partial t} + \frac{\partial\mathcal{H}}{\partial\mathfrak{g}}\dot{\mathfrak{g}} + \frac{\partial\mathcal{H}}{\partial\mathfrak{p}}\dot{\mathfrak{p}}\\
                & = \frac{\partial\mathcal{H}}{\partial t} + \left\langle \frac{\partial\mathcal{H}}{\partial\mathfrak{g}}, \dot{\mathfrak{g}} \right\rangle + \left\langle \frac{\partial\mathcal{H}}{\partial\mathfrak{p}}, \dot{\mathfrak{p}} \right\rangle. 
            \end{aligned}
        \end{equation}
        Let $\boldsymbol{\eta} = \frac{\partial\mathcal{H}}{\partial\mathfrak{g}}$, then using (\ref{eq:hamil-dyn-p}) and by comparing (\ref{eq:lie-state}) with (\ref{eq:hamil-dyn-g}) \cite[Equation. 19]{duong_hamiltonian_2021}, it holds that:
        \begin{equation}
            \begin{aligned}
                & = \frac{\partial\mathcal{H}}{\partial t} + \langle \boldsymbol{\eta}, \mathsf{T}_e\mathsf{L}_{\mathfrak{g}}(\hat{\boldsymbol{\xi}}) \rangle + \langle \hat{\boldsymbol{\xi}}, \boldsymbol{\eta} \rangle \\
                & = \frac{\partial\mathcal{H}}{\partial t} + \langle \boldsymbol{\eta}, \mathsf{T}_e\mathsf{L}_{\mathfrak{g}}(\hat{\boldsymbol{\xi}}) \rangle - \left\langle \hat{\boldsymbol{\xi}}, \mathsf{T}_e\mathsf{L}^\ast_{\mathfrak{g}}\left( \frac{\partial\mathcal{H}}{\partial\mathfrak{g}}\right)\right\rangle + \langle\hat{\boldsymbol{\xi}}, \mathtt{ad}^\ast_{\hat{\boldsymbol{\xi}}}(\mathfrak{p}) \rangle + \langle \hat{\boldsymbol{\xi}}, \mathbf{B}(\mathfrak{q})\mathbf{u} \rangle \\
                & = \frac{\partial\mathcal{H}}{\partial t} + \langle \boldsymbol{\eta}, \mathsf{T}_e\mathsf{L}_{\mathfrak{g}}(\hat{\boldsymbol{\xi}}) \rangle - \left\langle \hat{\boldsymbol{\xi}}, \mathsf{T}_e\mathsf{L}^\ast_{\mathfrak{g}}\left( \boldsymbol{\eta}\right)\right\rangle + \langle\hat{\boldsymbol{\xi}}, \mathtt{ad}^\ast_{\hat{\boldsymbol{\xi}}}(\mathfrak{p}) \rangle + \langle \hat{\boldsymbol{\xi}}, \mathbf{B}(\mathfrak{q})\mathbf{u} \rangle.
            \end{aligned}
        \end{equation}
        Using \cite[Definition. 9]{duong_hamiltonian_2021} and \ref{eq:adjoint}:
        \begin{equation}
            \langle\hat{\boldsymbol{\xi}}, \mathtt{ad}^\ast_{\hat{\boldsymbol{\xi}}}(\mathfrak{p}) \rangle = \langle \mathtt{ad}_{\hat{\boldsymbol{\xi}}}(\hat{\boldsymbol{\xi}}), \mathfrak{p} \rangle = \langle [\hat{\boldsymbol{\xi}}, \hat{\boldsymbol{\xi}}], \mathfrak{p}\rangle = 0
        \end{equation}
        and Definition (\ref{def:dual-map}):
        \begin{equation}
            \begin{aligned}
                & = \frac{\partial\mathcal{H}}{\partial t} + \langle \boldsymbol{\eta}, \mathsf{T}_e\mathsf{L}_{\mathfrak{g}}(\hat{\boldsymbol{\xi}}) \rangle - \langle \boldsymbol{\eta}, \mathsf{T}_e\mathsf{L}_{\mathfrak{g}}(\hat{\boldsymbol{\xi}}) \rangle + \langle \hat{\boldsymbol{\xi}}, \mathbf{B}(\mathfrak{q})\mathbf{u} \rangle \\
                & = \frac{\partial\mathcal{H}}{\partial t} + \langle \hat{\boldsymbol{\xi}}, \mathbf{B}(\mathfrak{q})\mathbf{u} \rangle.
            \end{aligned}
        \end{equation}
        As the Hamiltonian is independent of time (as both the kinetic and potential energies are not evolving as a function of time for non-relativistic systems \cite{mechanics_1960}), $\frac{\partial\mathcal{H}}{\partial t} = 0$ and the rate of change for the Hamiltonian becomes:
        \begin{equation}
            \frac{\text{d}\mathcal{H}}{\text{d}t} = \langle \hat{\boldsymbol{\xi}}, \mathbf{B}(\mathfrak{q})\mathbf{u} \rangle.
        \end{equation}
        Therefore, if there is zero input to the system, $\frac{\text{d}\mathcal{H}}{\text{d}t} = 0 $ and energy is conserved.
    \end{proof}
    
\section{$\mathtt{SE}(3)$ Metric Tensor and Geodesics}\label{app:metric-tensor}

The choice of the metric tensor $\mathbb{G}_{\gamma(t)}$ is dictated by the lack of a bi-invariant metric of $\mathtt{SE}(3)$ \cite[Corollary. A.5.1]{math_intro_manip_1994}. However, we can choose the kinetic energy metric based on the physics of the multi-body system, which is related to both the motion in task space as well as the configuration of the robot $\mathfrak{q} \in \mathcal{Q}$ \cite{geo_control_2005}:
\begin{equation}
    \mathbb{G}_{\text{KE}} = \frac{1}{2}\mathbb{G}_{\mathcal{L}}(\mathfrak{q}) \hat{\boldsymbol{\xi}}_b^\top\hat{\boldsymbol{\xi}}_b,
\end{equation}
where $\mathbb{G}_{\mathcal{L}}(\mathfrak{q}) \equiv \tilde{\mathbf{M}}(\mathfrak{q})$ represents the mass-matrix tensor for the Lagrangian kinetic energy \cite{pontryagin_riemannian_2015}. Computing the geodesic on $\mathtt{SE}(3)$ proves a more complex task, as a closed-form solution for general Lie groups does not currently exist \cite{canzini_iros_2024}. However, we can exploit the underlying physics of the Hamiltonian dynamics to obtain an expression for the geodesic in the form of Hamiltonian flows. Hamilton's principle of least action gives us the energy flow across the cotangent bundle
\begin{equation} \label{eq:hamil-KE}
    E_{\mathcal{H}}[t] = \int_{I} \frac{1}{2}\mathbb{G}^{-1}_{\mathcal{L}}(\mathfrak{q})\mathfrak{p}^\top\mathfrak{p}\, \text{d}t.
\end{equation}
This action functional allows us to define the cogeodesic flow on the cotangent space. Let $\gamma(t) = \begin{bmatrix}
    \mathfrak{g}(t),\mathfrak{p}(t)
\end{bmatrix}$ be the coordinates for the geodesic $\gamma(t) : I \rightarrow \mathsf{T}^\ast\mathtt{SE}(3)$ on the interval $I \in [\mathfrak{g}_0, \mathfrak{g}_1]$. We can use the derivations of the geometric features in Section \ref{sec:sub:ph-dynamics} to find the path velocities:
\begin{equation} \label{eq:path_velocities}
    \dot{\gamma}(t) = \begin{bmatrix}
        \dot{\mathfrak{g}}(t) \\
        \dot{\mathfrak{p}}(t)
    \end{bmatrix} = \begin{bmatrix}
        \mathfrak{g}\hat{\boldsymbol{\xi}}_b \\
        \mathtt{ad}^*_{\hat{\boldsymbol{\xi}}}(\mathfrak{p}) - \mathsf{T}^\ast_e\mathsf{L}_\mathfrak{g} \left ( \frac{\partial \mathcal{H}(\mathfrak{g}, \mathfrak{p})}{\partial \mathfrak{g}} \right ) + \mathbf{B}(\mathfrak{q})\mathbf{u}
    \end{bmatrix}.
\end{equation}
Then (\ref{eq:path_velocities}) provides us with the geodesic distance:
\begin{equation} \label{eq:geodesic}
    d(\mathfrak{g}_0, \mathfrak{g}_1) = \inf_{\gamma \in PC([0,1]; \mathtt{SE}(3))} \int_{I} \sqrt{\begin{bmatrix}
    \dot{\mathfrak{g}} \\
    \dot{\mathfrak{p}}
    \end{bmatrix}^\top \mathbb{G}_{\mathcal{H}}\begin{bmatrix}
        \dot{\mathfrak{g}} \\
        \dot{\mathfrak{p}}
        \end{bmatrix}} \text{d}t
\end{equation}
for all piece-wise continuous curves on the interval $I \in [\mathfrak{g}_0, \mathfrak{g}_1]$, where $\mathbb{G}_\mathcal{H} = \mathbb{G}^{-1}_{\mathcal{L}} = \begin{bmatrix}
   	\mathbb{G}_\mathbf{p} & \mathbf{0} \\
   	\mathbf{0} & \mathbb{G}_\mathbf{R}
\end{bmatrix}$ is the Hamiltonian metric tensor. For an approximation of this distance at runtime, as we have a smooth Riemannian metric space, we formulate the left-invariant distance metric between two poses in $\mathtt{SE}(3)$ as \cite{lie_metrics_2024}:
\begin{equation}
   	d(\mathfrak{g}_0, \mathfrak{g}_1) \approx \sqrt{\delta_{\mathbf{R}}(\mathbf{R}_0, \mathbf{R}_1) + \delta_{\mathbf{p}}(\mathbf{p}_0, \mathbf{p}_1)},
\end{equation}
where:
\begin{align}
   	\delta_{\mathbf{R}}(\mathbf{R}_0, \mathbf{R}_1) = & \lVert\mathbb{G}_{\mathbf{R}}\lVert_F\; \cos^{-1}\left(\frac{\text{tr}((\mathbf{R}_0^\top\mathbf{R}_1 )) -1}{2}\right), \label{eq:delta-R} \\
   	\delta_{\mathbf{p}}(\mathbf{p_0}, \mathbf{p}_1) = & \lVert \mathbb{G}_\mathbf{P}(\mathbf{p}_0 - \mathbf{p}_1) \lVert_2, \label{eq:delta-p}
\end{align}
where $\lVert \cdot \lVert_2$ is the vector 2-norm and $\lVert \cdot \lVert_F$ is the Frobenius norm. 

\section{Experimental Details}\label{app:exp-details}

The four poses that are used in our experimentation were generated by randomly sampling the joint space of the robot to produce poses where the feature is in the camera's field of view. This was done to ensure that the estimated transport map would include the feature in its computation. The four poses are presented here in their homogeneous transformation form:
\begin{equation}
	\begin{aligned}
		\begin{matrix}
			\mathfrak{g}_1 = \begin{bmatrix}
				-0.63 & -0.47 & -0.61 & -0.16 \\
				-0.25 & 0.88 & -0.42 & 0.49 \\
				0.73 & -0.11 & -0.67 & 0.34 \\
				0 & 0 & 0 & 1 \\
			\end{bmatrix}, & \mathfrak{g}_2 = \begin{bmatrix}
				0.05 & -0.816 & -0.58 & -0.07 \\
				-0.85 & 0.27 & -0.45 & 0.54 \\
				0.52 & 0.51 & -0.68 & 0.15 \\
				0 & 0 & 0 & 1 \\
			\end{bmatrix}
		\end{matrix},\\
		\begin{matrix}
			\mathfrak{g}_3 = \begin{bmatrix}
				-0.18 & -0.62 & -0.76 & -0.1 \\
				-0.98 & 0.2 & -0.08 & 0.34 \\
				0.1 & 0.76 & -0.64 & 0.31 \\
				0 & 0 & 0 & 1 \\
			\end{bmatrix}, & \mathfrak{g}_4 = \begin{bmatrix}
				0.28 & 0.25 & -0.93 & -0.05 \\
				0.65 & 0.66 & 0.38 & 0.21 \\
				0.7 & -0.71 & 0.03 & 0.06 \\
				0 & 0 & 0 & 1 \\
			\end{bmatrix}
		\end{matrix}
	\end{aligned}
\end{equation}
The target pose is shown here in the homogenous transformation form:
\begin{equation}
	\mathfrak{g}^\ast = \begin{bmatrix}
		0.01 & 0.07 & -1.0 & -0.27 \\
		0.01 & 1.0 & 0.07 & 0.31 \\
		1.0 & -0.01 & 0.01 & 0.038 \\
		0 & 0 & 0 & 1
	\end{bmatrix}
\end{equation}
All the distances for the position vector in the homogenous transforms are given in meters. 
\end{document}